\documentclass{article} 
\usepackage{iclr2027_conference,times}
\iclrfinalcopy

\usepackage{amsmath,amsfonts,bm}

\def\eqref#1{equation~\ref{#1}}

\def\1{\bm{1}}

\DeclareMathAlphabet{\mathsfit}{\encodingdefault}{\sfdefault}{m}{sl}
\SetMathAlphabet{\mathsfit}{bold}{\encodingdefault}{\sfdefault}{bx}{n}

\usepackage[utf8]{inputenc} 
\usepackage[T1]{fontenc}    
\usepackage{hyperref}       
\usepackage{url}            
\usepackage{booktabs}       
\usepackage{amsfonts}       
\usepackage{nicefrac}       
\usepackage{silence}        
\usepackage{microtype}      
\usepackage{xcolor}         

\usepackage{graphicx}
\usepackage{amssymb}
\usepackage{amsmath}
\usepackage{amsthm}
\usepackage{algorithm}
\newtheorem{lemma}{Lemma}
\newtheorem{proposition}{Proposition}

\newtheorem{remark}{Remark}

\title{PALM: Physically-Anchored Latent-timescale Mamba for Stable Representations under Distribution Shift}

\author{{Sangyoon Bae} \\
Interdisciplinary Program in Artificial Intelligence \\
Seoul National University\\
Seoul, 08826, South Korea \\
\texttt{stellasybae@snu.ac.kr} \\
\And
Shinjae Yoo \\
Computational Science Initiative \\
Brookhaven National Laboratory \\
Upton, New York, 11973, USA  \\
\texttt{sjyoo@bnl.gov}
\And
Jiook Cha \\
Department of Psychology \\
Seoul National University \\
Seoul, 08826, South Korea  \\
\texttt{connectome@snu.ac.kr}}

\begin{document}

\maketitle

\begin{abstract}
Scientific foundation models are built to reuse representations across datasets and acquisition protocols, yet their representations degrade under distribution shift. We trace one source of this failure to selective state-space models (SSMs). Their memory depends on the step $\Delta$ and the state matrix $A$ only through their product, so data do not identify the physical timescale of memory, and under a shift the resulting memory no longer matches the signal's timescales. We call this failure \emph{memory-timescale misalignment}. To resolve it, the split between $\Delta$ and $A$ must be tied to physical units. We therefore propose \textbf{Physically-Anchored Latent-timescale Mamba (PALM)}, which disentangles $A$ from $\Delta$ by anchoring both to acquisition units: $\Delta$ through spectrum-estimated knee frequencies and $A$ through a physical memory range. We evaluate PALM against parameter-matched baselines in two domains, under temporal-context truncation, low-resource transfer, and brain-state shift on fMRI and under held-out-station forecasting on Weather-5K. Under fMRI context truncation with run-level spectral calibration, PALM keeps its latent timescales in physical units and retains $11\%$ higher representation similarity than Mamba2. At geographically isolated weather stations never seen during training, PALM achieves the lowest error on every variable and horizon, $15.2\%$ below the best baseline on average, and outperforms zero-shot forecasting foundation models. Across both domains, the gains come from where the backbone places memory in physical time rather than from model size. We therefore argue that physically anchored latent timescales are a design requirement for scientific foundation models that must transfer across acquisition conditions.
\end{abstract}

\section{Introduction}
\label{sec:introduction}

Scientific foundation models are built to reuse representations across datasets, acquisition protocols, and tasks~\citep{tak2026generalizable,wang2025towards,wang2025slim,wang2026omni,abhimanyu2024decoder,ansari2024chronos,woo2024unified}. In deployment, their representations degrade under distribution shift, when temporal resolution, scanner or site, brain state, season, station, or label budget change. Recent brain and forecasting foundation models mainly scale data, parameters, or adaptation. We instead ask what in the backbone itself fails under shift. Neural and atmospheric signals organize their predictive structure across interacting physical timescales. fMRI exhibits scale-free, multifractal dynamics~\citep{he2011scale,ciuciu2012scale,bae2026learning}, and stratified atmospheric turbulence shows distinct spectral regimes across physical scales~\citep{cheng2020model}. A reusable backbone must place its memory on these timescales, not only fit the correlations of one split.

We trace one source of the degradation to selective state-space models (SSMs)~\citep{gu2023mamba}. A selective SSM updates its state with the decay $\exp(\Delta A)$, so its memory depends on the step $\Delta$ and the state matrix $A$ only through their product. Data therefore fix $\Delta|A|$ but not how it divides between $\Delta$ and $A$, because the rescaling $(\Delta, A)\to(c\Delta,\,c^{-1}A)$ leaves the product unchanged (Proposition~\ref{prop:gauge}). $\Delta$ is not the sampling interval either, so neither factor carries a physical unit.

This missing reference matters under shift. $\Delta$ changes with the new input, such as a truncated context or a shifted spectrum, while $A$ keeps its learned value. The resulting decay then follows an arbitrary division learned in distribution, and the model's memory no longer matches the signal's timescales. We call this failure \textbf{memory-timescale misalignment}.

To resolve it, the split between $\Delta$ and $A$ must be tied to physical units. We propose \textbf{Physically-Anchored Latent-timescale Mamba (PALM)}, which anchors both factors. $\Delta$ has no unit of its own, so PALM gives it a physical reference from the signal's spectrum. Following frequency-resolved multifractal fMRI modeling~\citep{bae2026learning}, we fit a piecewise power law to the spectrum,
\begin{equation}
P(f) \propto f^{-\beta_k},
\quad f \in [f_{k-1}, f_k],
\quad k = 1, \dots, K,
\label{eq:piecewise_powerlaw}
\end{equation}
where the knee frequencies $f_k$ mark transitions between temporal regimes. PALM calibrates the scale-specific steps $\Delta_k$ against the acquisition unit (TRs or hours) at these knees. This calibration does not turn $\Delta$ into the sampling interval. It sets the scale of the latent coordinate so that $\Delta$ and $\tau$ share the acquisition unit. PALM anchors $A$ by initializing $\tau=1/|A|$ across the acquisition-time range, and a coupling penalty keeps training from undoing either anchor. The two anchors penalize the scaling symmetry from opposite ends, so $\tau$ stays referenced to physical units and comparable across conditions. We build on the established view that dynamics are organized on a logarithmic ladder of timescales~\citep{lindeberg1994scale,hasson2015hierarchical,tallec2018can,jacques2022time}. Our contribution is the disentanglement of $A$ from $\Delta$ and the dual anchoring.

Our main contributions are as follows:

\begin{itemize}
\item \textbf{Gauge non-identifiability of physical timescale.} We show that a selective SSM identifies only the product $\Delta|A|$ and leaves the split between $\Delta$ and $A$ unreferenced (Proposition~\ref{prop:gauge}). This result identifies memory-timescale misalignment as a backbone-level source of degradation under shift.
\item \textbf{Dual physical anchoring.} We propose PALM, which calibrates $\Delta$ against the acquisition unit at spectrum-estimated knees and anchors $A$ through its memory timescale $\tau=1/|A|$. A coupling penalty keeps training from undoing either anchor, so $\tau$ stays comparable across conditions.
\item \textbf{Cross-domain validation.} We validate PALM on brain imaging and weather forecasting under four structured shifts against parameter-matched baselines. PALM keeps representations more stable under context truncation and achieves the lowest forecasting error at unseen stations, outperforming zero-shot forecasting foundation models.
\end{itemize}

\section{Related Work}
\label{sec:related}
Foundation models for scientific time series advance representation learning through larger corpora, generic objectives, and adaptation modules~\citep{tak2026generalizable,wang2025towards,wang2025slim,wang2026omni,abhimanyu2024decoder,ansari2024chronos,woo2024unified}. PALM is complementary and asks whether the backbone itself preserves physically meaningful temporal scales under shift. Scale-space theory, hierarchical temporal-receptive-window neuroscience, and multi-timescale RNN/SSM analyses establish that dynamics span a (roughly logarithmic) ladder of timescales~\citep{lindeberg1994scale,hasson2015hierarchical,tallec2018can,jacques2022time,gu2020hippo}, and PALM builds on this view. Recent multi-scale SSM/Mamba variants~\citep{Karami2025MSSSMAM,karadag2025ms,chen2025stm3,li2025multi} set several sampling rates or discretization values but leave $\Delta$ entangled with $A$. The physical timescale of their memory therefore stays unreferenced. PALM differs from these variants in identifiability, not in multi-scale capacity. PALM places scales at measured spectral knees and anchors \emph{both} $\Delta$ and $A$ to physical units. Dual anchoring acts only on the $(\Delta, A)$ parameterization, so it is complementary to these architectures. Appendix~\ref{sec:appx:related} discusses each method.

\section{Evaluation Setting}
\label{sec:problem}

If misalignment is a failure of the backbone, correcting it should help wherever acquisition conditions change, not only in one modality. We therefore test PALM on brain imaging recorded in seconds and on weather stations recorded in hours, under four shifts that keep the task fixed:
\textbf{Axis (i)} temporal-context truncation of task motor blocks with full-run spectral calibration;
\textbf{Axis (ii)} low-resource transfer;
\textbf{Axis (iii)} brain-state shift (resting $\to$ task);
\textbf{Axis (iv)} Weather-5K spatial out-of-distribution (OOD) forecasting at geographically isolated stations never seen during training.
We report downstream performance and representation drift (CKA/dCor), and the formal setup is in Appendix~\ref{sec:appx:problem_setup}.

\section{Theory}
\label{sec:theory}

A selective SSM's effective memory is a weighting over past lags that forms the current representation. Fast decay emphasizes recent samples, and slow decay preserves longer context~\citep{oppenheim1999discrete,brown1959statistical}. How far back the memory reaches is set by the timescale $\tau=1/|A|$ together with the step $\Delta$.
When TR, context length, brain state, or station regime changes, the task-relevant lag structure changes, and a model whose memory sits at the wrong physical timescale produces shifted latent states.
This section formalizes why data leave the split between $\Delta$ and $A$ unreferenced (Section~\ref{sec:theory:gauge}) and how the resulting misalignment propagates to representation error (Section~\ref{sec:theory:kernelbound}).

\subsection{Physical Timescale is a Gauge Freedom of the In-Distribution Objective}
\label{sec:theory:gauge}

The Mamba-2/SSD backbone uses a diagonal (per-head scalar) state matrix $A<0$ and step $\Delta>0$, giving the recurrence the model actually runs:
\begin{equation}
h_k = e^{\Delta A}\, h_{k-1} + (\Delta B)\, x_k,
\qquad y_k = C h_k,
\label{eq:ssd_recurrence}
\end{equation}
so the continuous memory constant is $\tau = 1/|A|$ and the per-sample decay is $e^{\Delta A} = e^{-\Delta/\tau}$. The following statement makes precise why $\Delta$ and $A$ cannot be separated from data alone.

\begin{proposition}[Gauge non-identifiability of physical timescale]
\label{prop:gauge}
For a condition-independent scalar step $\Delta$ and any $c>0$, the reparameterization $(\Delta, A, B) \mapsto (c\Delta,\, c^{-1}A,\, c^{-1}B)$ leaves the input--output map of Equation~\ref{eq:ssd_recurrence} exactly invariant. Consequently, from in-distribution data one can identify only the dimensionless per-sample decay $\Delta|A|$ (equivalently, the timescale ratios $\tau_i/\tau_j$); the absolute physical timescale $\tau = 1/|A|$ and the step $\Delta$ are individually unidentifiable, fixed only up to a common multiplicative scale.
\end{proposition}

\begin{proof}
$e^{(c\Delta)(c^{-1}A)} = e^{\Delta A}$ and $(c\Delta)(c^{-1}B) = \Delta B$; $C$ and the inputs $x_k$ are unchanged, so every state $h_k$ and output $y_k$ is unchanged. For diagonal $A$ the shared scalar step $\Delta$ makes the same $c$ act on every coordinate, so all $\tau_i$ scale by $c$ while every ratio $\tau_i/\tau_j$ and every product $\Delta|A_i|$ is fixed.
\end{proof}

For input-dependent steps, as in Mamba-2, and under weight decay, the invariance holds only approximately.

\textbf{Identifiability, not expressivity.} For a single fixed acquisition condition, the transformation is an exact symmetry of the whole model, so gauge fixing does not change the in-distribution function class. The proposition instead establishes that the physical timescale is a \emph{flat direction} of the objective. Training receives no gradient to place $\tau$ (or $\Delta$) at any particular physical value. For this reason $\Delta$ is not a time-step despite its name, and prior work tunes the $\Delta$-range as a hyperparameter~\citep{gu2020hippo}. The symmetry by itself cannot cause drift, because every point on the orbit computes the same function on every input, including shifted ones. The split matters when one factor is set per input or condition and the other is shared. In PALM, $\Delta_k$ is recomputed from each input's spectrum and $A$ is shared, so the split of the product decides which physical memory a change in $\Delta_k$ corresponds to. PALM fixes the split by anchoring $\Delta$ to spectrum-derived knees and $A$ to a physical memory range. Applying the global $c$ of Proposition~\ref{prop:gauge} then incurs a penalty on at least one anchor. That fixing the gauge \emph{helps} under shift is an empirical claim, and we support it in Section~\ref{sec:results}.

\begin{remark}[Gauge fixing by dual anchoring]
\label{cor:gauge_fix}
Dual anchoring breaks the global scaling symmetry of Proposition~\ref{prop:gauge}. The penalties $\lambda_\Delta\mathcal{L}_\Delta$, with $\mathcal{L}_\Delta = \mathbb{E}[\log^2(\Delta/\mathrm{TR})]$, and $\lambda_A\mathcal{L}_A$ (Section~\ref{sec:methods}) penalize displacement along the gauge orbit and select a physically referenced representative, in which the steps $\Delta_k$ and the memory constant $\tau = 1/|A|$ are expressed relative to the acquisition unit (TR). The penalties are soft, so they bias training toward this representative rather than fixing $\Delta$ exactly. This is the sense in which PALM disentangles $A$ from $\Delta$.
\end{remark}

\subsection{Memory as Kernel Matching}
\label{sec:theory:kernelbound}

Viewing the encoder locally as a convolution, we bound the representation error by the $L^1$ mismatch between the true and model-induced kernels, $\|h-\tilde h\|_\infty \le M\,\|g-\tilde g\|_1$ (Lemma~\ref{lem:kernel_mismatch}). Memory-timescale misalignment therefore propagates to downstream error under shift. Multi-scale kernels lower the achievable mismatch (Appendices~\ref{sec:appx:proofs} and~\ref{sec:appx:kernel_approx}).

\section{Methods}
\label{sec:methods}

\subsection{Overall Architecture}
\label{sec:methods:overview}

\begin{figure}[t]
\centering
\includegraphics[width=0.85\linewidth]{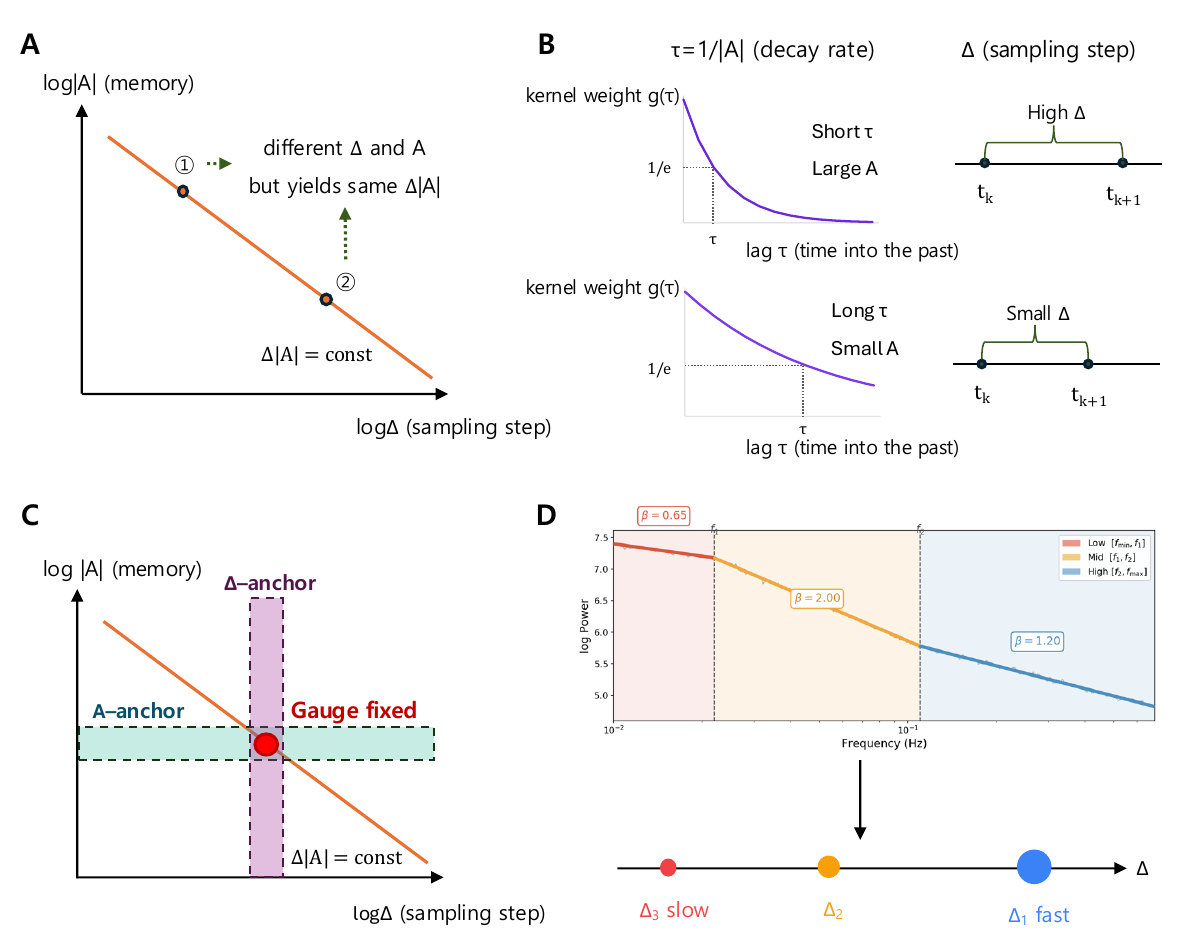}
\caption{\textbf{The $\Delta$--$A$ gauge and dual physical anchoring.}
\textbf{(A)} In the $(\log\Delta, \log|A|)$ plane, models with the same product $\Delta|A|$ lie on one orbit (points \textcircled{1},\textcircled{2}) and produce identical outputs; the physical timescale $\tau=1/|A|$ varies freely along it.
\textbf{(B)} $A=-1/\tau$ sets the memory kernel's decay rate (short $\tau$/large $|A|$, fast decay; long $\tau$/small $|A|$, slow; $\tau$ is the $1/e$ lag); $\Delta$ sets the discretization step between consecutive states $t_k,t_{k+1}$ --- a model coordinate, not the data sampling interval.
\textbf{(C)} The $\Delta$-anchor (vertical) and $A$-anchor (horizontal band) meet the orbit near a single point, selecting $\tau$.
\textbf{(D)} Knee frequencies of the piecewise-power-law PSD map to ordered steps $\Delta_1{>}\Delta_2{>}\Delta_3$: the highest-frequency band takes the largest step (fast scale $\Delta_1$) down to the lowest-frequency band with the smallest step and longest memory (slow scale $\Delta_3$), in acquisition units; $A$ stays a single anchored range.}
\label{fig:architecture_multiscale}
\vspace{-6pt}
\end{figure}

Figure~\ref{fig:architecture_multiscale}A--C shows the gauge that PALM fixes. Models with the same product $\Delta|A|$ lie on one orbit and compute the same function, so the latent timescale $\tau=1/|A|$ can move freely along the orbit (A). $A=-1/\tau$ sets how fast the memory kernel decays, and $\Delta$ sets the discretization step between consecutive states (B). PALM places one anchor on each factor, and the two anchors meet the orbit near a single point, which selects $\tau$ in acquisition units (C). Figure~\ref{fig:architecture_multiscale}D shows how PALM places multiple scales on the signal's own physical timescales, one per spectral regime. SpectralHyperNet estimates knee frequencies and scaling exponents from the input spectrum, and these estimates parameterize ordered SSM steps $\Delta_k$ in the same acquisition units as the data. Fast-to-slow SSM head groups process the original sequence, and cross-scale attention models interactions among their states.

\subsection{Datasets}
\label{sec:methods:data}

Each shift axis uses one dataset, and Appendix~\ref{sec:appx:data} gives the details. \textbf{HCP motor fMRI} (360 cortical ROIs, TR$=0.72$\,s, five motor conditions) supports temporal-context truncation and low-resource transfer. \textbf{HCP resting-state fMRI} supports brain-state shift, with training on resting-state and testing on motor-task fMRI. \textbf{Weather-5K}~\citep{han2024far} supports hourly forecasting at about 10 geographically isolated stations held out from a 100-station subset.

\subsection{SpectralHyperNet: Spectrum-Guided Scale Inference}

We model the power spectral density of each multivariate time series using the piecewise power-law form in Equation~\ref{eq:piecewise_powerlaw}, with a configurable number of frequency regimes. The main experiments use the default $K=3$ setting, separated by knee frequencies $(f_1, f_2)$ (sensitivity to $K$ and knee estimation in Appendix~\ref{sec:appx:k_sensitivity}).
We empirically verify this piecewise structure across datasets and domains in Appendix~\ref{sec:appx:psd_multifractal}.

\textbf{Spectral modeling.} Given a spectral context sequence $x^{\mathrm{spec}}$, we estimate the power spectral density $P(f)$ via FFT and fit Equation~\ref{eq:piecewise_powerlaw}. By default, $x^{\mathrm{spec}}$ is the same sequence processed by the SSM backbone; in HCP early-window experiments, SpectralHyperNet receives the 284-TR motor run while the SSM backbone receives only the first 3 or 2 TRs of the event-locked motor block. This tests whether full-run spectral structure stabilizes shortened task evidence, not a changed scanner TR.

\textbf{Knee estimation.} SpectralHyperNet predicts the $K-1$ knee frequencies and $K$ scaling exponents from the spectrum, initialized from a data-driven scipy fit on the first training batch. A multi-resolution MAE spectral fit loss drives knee placement, and a seam loss encourages piecewise continuity. Appendix~\ref{sec:appx:training_details} gives the hyperparameters.

\paragraph{Frequency-to-delta mapping.}
Each band is represented by its energy centroid $f_{c,k}$ under the fitted slope (Appendix~\ref{sec:appx:arch_details}). The core principle is $\Delta_k \propto f_{c,k}$. Higher-frequency bands receive larger $\Delta$ (faster updates), and lower-frequency bands receive smaller $\Delta$ (longer memory). A hierarchical constraint $\Delta_1 \ge \Delta_2 \ge \Delta_3$ then enforces a fast-to-slow ordering. Algorithm~\ref{alg:spectrum_to_delta} (Appendix~\ref{sec:appx:arch_details}) gives the exact mixing formula, and Appendix~\ref{sec:appx:delta_ablation} ablates its parameters.

\subsection{Dual Physical Anchoring}
\label{sec:methods:anchoring}

\textbf{Why both factors need an anchor.} Leaving $\Delta$ and $A$ both unconstrained lets training move along the gauge orbit of Proposition~\ref{prop:gauge} and undo any physical grounding. A hard constraint on either factor would already fix the orbit. With soft penalties, anchoring $\Delta$ alone is insufficient, because the objective pins only the product $\Delta|A|$. In weak-signal or truncated-window regimes even that product is loosely determined, so a soft $\Delta$-penalty can be traded off against $A$ along the flat direction. $\tau=1/|A|$ then drifts away from the acquisition-time range while $\Delta$ stays near its knees, which is the signature of memory-timescale misalignment. PALM therefore anchors \emph{both} factors.

\textbf{Anchoring $\Delta$.} The spectrum-to-delta map expresses each $\Delta_k$ in dataset-native acquisition units (TRs, hours) at the spectrum-derived knees. A log-scale regularizer $\lambda_\Delta\mathcal{L}_\Delta$, with $\mathcal{L}_\Delta = \mathbb{E}[\log^2(\Delta/\mathrm{TR})]$, discourages arbitrary rescaling.

\textbf{Anchoring $A$.} $\tau$ is the flat coordinate of Proposition~\ref{prop:gauge}, so it needs its own physical prior. PALM parameterizes $A=-\exp(A_{\log})$ with $A\approx -1/\tau$ and samples $\tau$ from a log-uniform acquisition-time range, so $A$ carries physical units at initialization. During training, an $A$-scale penalty
\[
\mathcal{L}_{A}
= \sum_{b} \left(\log \overline{|A|}_b - \log a_{0,b}\right)^2,
\qquad
\overline{|A|}_b=\mathrm{mean}_{j\in b}\,\exp(A_{\log,j}),
\]
weighted by $\lambda_A$, keeps each block $b$'s average decay scale near its initialization $a_{0,b}$.

\textbf{Coupling.} The two anchors penalize the shared scaling $c$ of Proposition~\ref{prop:gauge} from opposite ends. Rescaling $A$ to undo the spectrum-derived $\Delta_k$ is therefore penalized, and $\tau = 1/|A|$ stays referenced to physical time units (Remark~\ref{cor:gauge_fix}). Section~\ref{sec:results:mechanism} reports the realized latent timescales with and without the anchoring, and Appendix~\ref{sec:appx:arch_details} gives implementation details.

\subsection{Multi-Scale State-Space Block}

Each PALM block extends the Mamba-2 Structured State Space Duality (SSD) kernel~\citep{dao2024transformers} to a multi-scale setting. Its multi-head structure is a natural unit for assigning different discretization to different head groups. PALM shares $\Delta_k$ within a scale group, and each head keeps an independent state transition. This design adds within-scale capacity without materially increasing parameters over single-scale Mamba2. Lightweight cross-scale attention then fuses the scale-specific states. Full equations and head counts are in Appendix~\ref{sec:appx:arch_details}.

\subsection{Backbone, Prediction Head, and Training}
\label{sec:methods:training}

The final hidden sequence is temporally pooled and passed to a task-specific prediction head for classification or forecasting. All models are trained under a shared protocol with matched parameter budgets across baselines. Appendix~\ref{sec:appx:arch_details} gives architecture dimensions and parameter matching. Appendices~\ref{sec:appx:training_details} and~\ref{sec:appx:data} give optimization settings, Weather-5K preprocessing, RevIN normalization~\citep{kim2021reversible}, and OOD split construction.

\section{Results}\label{sec:results}

Across all four deployment-motivated axes, comparisons use parameter-matched baselines (Mamba2, Transformer, 1D-CNN). In each domain, the baselines share PALM's parameter budget.

\subsection{Axis (i): Temporal-Context Shift --- HCP Motor Early-Window Truncation}
\label{sec:results:hcp_tr_shift}

\begin{figure}[t]
\centering
\includegraphics[width=\linewidth]{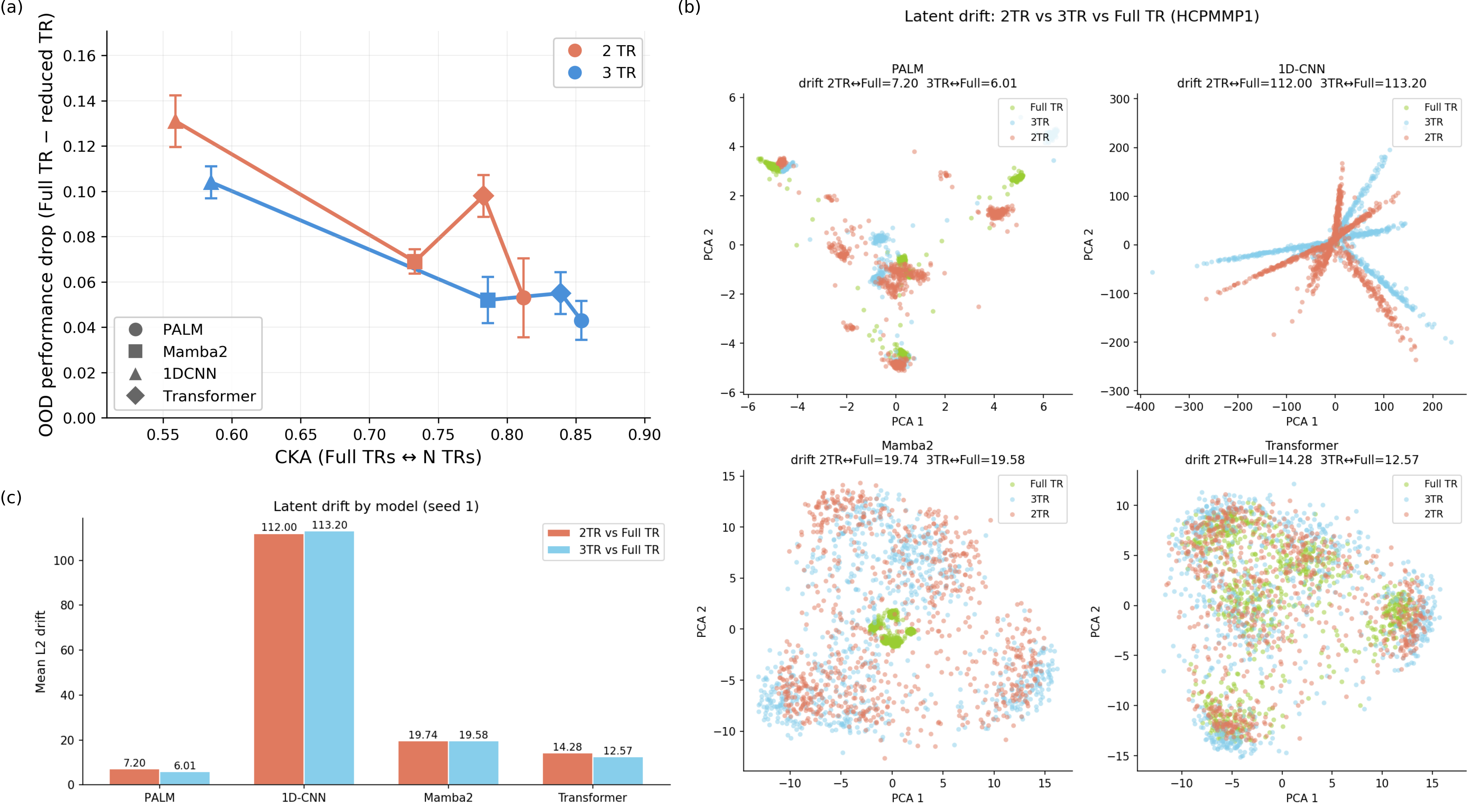}
\caption{
\textbf{Representation stability under temporal-context truncation (HCP motor task fMRI).}
\textbf{(a)} Relationship between representation similarity (CKA; full-block vs early-window models) and performance gap (full block $\rightarrow$ 2TR / 3TR).
\textbf{(b)} PCA visualization of latent embeddings under full-block, 3TR, and 2TR inputs.
\textbf{(c)} Mean $\ell_2$ latent drift between full-block and early-window conditions.
}
\label{fig:tr_shift_stability}
\end{figure}

Here separate early-window models are trained and evaluated with the SSM backbone restricted to the first 3 or 2 TRs of each event-locked motor block, while SpectralHyperNet still receives the full run for scale calibration (Section~\ref{sec:methods}). The design isolates whether run-level spectral structure stabilizes representations when task evidence is temporally shortened, holding task labels and subject identity fixed.

\paragraph{Full block $\rightarrow$ 2 TRs (severe truncation).}
PALM's full-block$\to$2TR accuracy drop is 22\% smaller than that of the best baseline (Table~\ref{tab:drift2tr}). Its representation similarity is 4\% higher on CKA and 5\% higher on dCor than the best baseline, and its CKA is 11\% higher than Mamba2's. PALM therefore preserves the latent geometry of the full-block solution when only the earliest time points remain. The ablations show that the spectrum-derived $\Delta$ matters for stability. Learnable $\Delta$ matches Mamba2's 2TR accuracy, and random $\Delta$ falls below it. Both replacements lower CKA/dCor and make them unstable across seeds. The next analysis separates this contribution from multi-scale capacity, and Section~\ref{sec:results:mechanism} measures the anchoring mechanism directly.

\begin{table}[t]
\centering
\footnotesize
\caption{
\textbf{HCP motor decoding under severe temporal-context truncation (full block $\rightarrow$ 2 TRs).}
Values are mean $\pm$ std over five seeds. CKA and dCor compare the full-block model with the separately trained 2TR model on the same test samples, and PALM calibrates on the full-run spectrum in both; the learnable/random $\Delta$ rows are PALM ablations in which the spectrum-derived $\Delta$ is replaced.
}
\label{tab:drift2tr}
\begin{tabular}{lcccc}
\toprule
Model & Full block & 2 TRs & \begin{tabular}[b]{@{}c@{}}CKA\\(Full block$\leftrightarrow$2TR)\end{tabular} & \begin{tabular}[b]{@{}c@{}}dCor\\(Full block$\leftrightarrow$2TR)\end{tabular} \\
\midrule
Mamba2 & $0.983\pm0.002$ & $0.914\pm0.005$ & $0.733\pm0.018$ & $0.770\pm0.011$ \\
1DCNN & $0.978\pm0.007$ & $0.847\pm0.009$ & $0.559\pm0.020$ & $0.622\pm0.012$ \\
Transformer & $0.982\pm0.006$ & $0.884\pm0.007$ & $0.783\pm0.015$ & $0.805\pm0.005$ \\
PALM (learnable $\Delta$) & $0.982\pm0.001$ & $0.913\pm0.030$ & $0.776\pm0.086$ & $0.804\pm0.078$ \\
PALM (random $\Delta$) & $0.983\pm0.004$ & $0.889\pm0.055$ & $0.737\pm0.105$ & $0.774\pm0.108$ \\
\midrule
\textbf{PALM} & $\mathbf{0.987\pm0.007}$ & $\mathbf{0.933\pm0.016}$ & $\mathbf{0.812\pm0.033}$ & $\mathbf{0.846\pm0.022}$ \\
\bottomrule
\end{tabular}
\end{table}

\paragraph{Where the gain comes from: causal calibration and the scale-space grid.}
We test whether the gain requires privileged access to the full test run or arises from multi-scale capacity alone, using four variants at 2\,TR (Table~\ref{tab:causal_variants}). The fixed log-spaced grid also stands in for concurrent multi-scale SSM variants (Section~\ref{sec:related}), which use multiple fixed discretization scales without measured placement. Running it inside the same backbone separates multi-scale capacity from physical anchoring under matched conditions. Multi-scale capacity provides most of the accuracy, and the four variants are statistically indistinguishable on accuracy (all pairwise paired 95\% CIs include zero). The accuracy benefit therefore does not depend on privileged spectral access. Representation stability discriminates the variants. Full-run PALM exceeds the log-spaced grid on \emph{every} seed, with 1.7\% higher CKA, and paired 95\% CIs exclude zero for both CKA and dCor (Appendix~\ref{sec:appx:hcp_tr_shift}). The grid alone already reaches a CKA close to full-run PALM, so multi-scale capacity accounts for much of the stability, and the spectrum-derived $\Delta$ adds a smaller increment that is consistent across seeds. The SpectralHyperNet, absent from the grid, contributes stability of representation geometry rather than accuracy.

\begin{table}[htbp]
\centering
\footnotesize
\caption{
\textbf{Decomposing the source of the gain (HCP motor, full block $\rightarrow$ 2\,TR, five seeds).}
Rows enable one factor at a time, from a fixed log-spaced multi-scale grid to the privileged full-run spectrum. Values are mean $\pm$ std.
}
\label{tab:causal_variants}
\begin{tabular}{llcc}
\toprule
Factor enabled & Variant & Params & 2\,TR accuracy \\
\midrule
multi-scale only (no spectral prior) & log-spaced grid & 5.2M & $0.928\pm0.010$ \\
\;+ spectrum-derived $\Delta$, \textbf{zero test access} & train-fixed knees & 5.98M & $0.930\pm0.010$ \\
\;+ spectrum-derived $\Delta$, same window & prefix-only & 5.98M & $0.932\pm0.010$ \\
\;+ privileged full-run spectrum & full-run adaptive & 5.98M & $0.932\pm0.010$ \\
\bottomrule
\end{tabular}
\end{table}

Under milder full-block$\rightarrow$3\,TR truncation, PALM again achieves the best accuracy and exceeds all baseline architectures on CKA and dCor. Drift-regularization and WeightWatcher analyses support the stability result (Appendices~\ref{sec:appx:hcp_tr_shift} and~\ref{sec:appx:weightwatcher}).

\subsection{Axis (ii): Low-Resource Generalization --- HCP Motor Data Scarcity}
\label{sec:results:data_scarce}

\begin{table}[t]
\centering
\footnotesize
\caption{
\textbf{Low-resource generalization on HCP motor decoding.}
Values are mean $\pm$ std over five seeds. Columns give the fraction of labeled training data.
}
\label{tab:data_scarce}
\begin{tabular}{lcccc}
\toprule
Model & 1\% & 5\% & 10\% & 100\% \\
\midrule
PALM & $\mathbf{0.797\pm0.040}$ & $0.940\pm0.017$ & $0.963\pm0.011$ & $\mathbf{0.987\pm0.007}$ \\
Mamba2 & $0.752\pm0.015$ & $0.934\pm0.022$ & $0.963\pm0.006$ & $0.983\pm0.002$ \\
1DCNN & $0.709\pm0.026$ & $0.927\pm0.018$ & $0.956\pm0.007$ & $0.978\pm0.007$ \\
Transformer & $0.774\pm0.031$ & $\mathbf{0.948\pm0.017}$ & $\mathbf{0.967\pm0.006}$ & $0.982\pm0.006$ \\
\bottomrule
\end{tabular}
\end{table}

Table~\ref{tab:data_scarce} reports motor decoding as the labeled training fraction varies from 1\% to 100\%, subsampling only the training set while keeping validation and test subjects fixed. The 1\% condition therefore tests generalization from a small labeled set to the same held-out distribution, not an easier test set. At 1--10\% labels, models are pretrained on resting-state fMRI and then finetuned, and the 100\% column reports training from scratch. PALM performs best in the most data-scarce 1\% regime, 3\% more accurate than the best baseline. At 5--10\%, the Transformer is slightly more accurate, and at 100\% PALM is again the most accurate. We read this as improved sample efficiency under extreme label scarcity rather than uniform dominance. Appendices~\ref{sec:appx:pretraining} and~\ref{sec:appx:pretrain} give the pretraining details and compare pretrained with from-scratch models.

\subsection{Axis (iii): Brain-State Shift --- Resting-State \texorpdfstring{$\rightarrow$}{to} Task Generalization}
\label{sec:results:brainstateshift}

\begin{table}[!h]
\centering
\footnotesize
\caption{
\textbf{Brain-state shift: resting-state training $\to$ motor-task test (Axis iii).}
Values are mean $\pm$ std over five seeds.
}
\label{tab:brainstateshift}
\begin{tabular}{lc}
\toprule
Model & Accuracy \\
\midrule
Mamba2      & $0.591\pm0.038$ \\
1D-CNN      & $0.589\pm0.040$ \\
Transformer & $0.554\pm0.063$ \\
\midrule
\textbf{PALM} & $\mathbf{0.632\pm0.013}$ \\
\bottomrule
\end{tabular}
\end{table}

We train from scratch on HCP resting-state fMRI for sex classification and evaluate the same label on motor-task fMRI without adaptation, testing whether the memory timescales learned from spontaneous dynamics remain useful under a task-evoked acquisition state.
Sex is used as the target because it is scan-condition independent and available for both scan types, so evaluation changes brain/acquisition state while keeping the label fixed (subject-split details in Appendix~\ref{sec:appx:data}).
We do not report CKA/dCor here because resting-state and task fMRI are not paired observations of the same input under two temporal conditions, so cross-condition similarity would conflate representation stability with differences in behavioral state and sequence content.
PALM is 7\% more accurate than the best baseline (Mamba2), with about one third of its cross-seed standard deviation (Table~\ref{tab:brainstateshift}). We read the lower variance as reduced sensitivity to the resting$\to$task acquisition-state shift, not only a higher mean score.

\subsection{Axis (iv): Spatial OOD --- Weather-5K Held-Out Stations}
\label{sec:results:weather}

This axis tests a geographic distribution shift: forecasting at weather stations \emph{held out from training}. It also applies the same spectrum-guided parameterization, without domain-specific adaptation, to a different modality and sampling regime (hourly stations, heterogeneous units). PALM achieves the lowest variable-wise MAE at every one of the 24 variable--horizon endpoints (Table~\ref{tab:weather}). It reduces MAE over the best non-PALM baseline by 15.2\% on average. Averaged over variables, the reduction grows with horizon, from 11.4\% at the shortest horizon to 17.9\% at the longest. Per endpoint, the reduction ranges from 0.8\% to 33.6\%.
\begin{table}[htbp]
\centering
\scriptsize
\setlength{\tabcolsep}{1.1pt}
\caption{
\textbf{Weather-5K held-out-station spatial-OOD forecasting (Axis iv).}
Values are mean $\pm$ std. Variable-wise spatial OOD MAE is reported for temperature (TMP), sea-level pressure (SLP), dew point (DEW), wind rate, and wind angle components (cos/sin) across 24h, 72h, 120h, and 168h forecasting horizons.
}
\label{tab:weather}
\begin{tabular}{llcccccc}
\toprule
Horizon & Model & TMP & SLP & DEW & Wind Rate & Wind Cos & Wind Sin \\
\midrule
24h & Mamba2      & $1.823\pm0.127$ & $2.486\pm0.379$ & $1.951\pm0.287$ & $1.903\pm0.009$ & $0.561\pm0.000$ & $0.530\pm0.020$ \\
    & 1D-CNN      & $1.150\pm0.066$ & $1.505\pm0.462$ & $1.316\pm0.213$ & $1.330\pm0.008$ & $0.416\pm0.000$ & $0.389\pm0.018$ \\
    & Transformer & $1.351\pm0.078$ & $2.066\pm0.648$ & $1.417\pm0.256$ & $1.537\pm0.009$ & $0.462\pm0.000$ & $0.450\pm0.018$ \\
    & \textbf{PALM} & $\mathbf{0.921\pm0.052}$ & $\mathbf{1.190\pm0.341}$ & $\mathbf{1.066\pm0.041}$ & $\mathbf{1.301\pm0.154}$ & $\mathbf{0.410\pm0.005}$ & $\mathbf{0.370\pm0.014}$ \\
\midrule
72h & Mamba2      & $1.827\pm0.533$ & $3.984\pm0.640$ & $2.513\pm0.121$ & $1.949\pm0.312$ & $0.587\pm0.011$ & $0.591\pm0.007$ \\
    & 1D-CNN      & $1.541\pm0.247$ & $2.931\pm0.776$ & $1.855\pm0.613$ & $1.569\pm0.353$ & $0.479\pm0.019$ & $0.482\pm0.077$ \\
    & Transformer & $2.266\pm0.678$ & $3.911\pm0.946$ & $2.017\pm0.163$ & $1.732\pm0.351$ & $0.533\pm0.012$ & $0.513\pm0.024$ \\
    & \textbf{PALM} & $\mathbf{1.167\pm0.127}$ & $\mathbf{2.214\pm0.826}$ & $\mathbf{1.411\pm0.108}$ & $\mathbf{1.512\pm0.204}$ & $\mathbf{0.475\pm0.013}$ & $\mathbf{0.439\pm0.018}$ \\
\midrule
120h & Mamba2      & $2.469\pm0.287$ & $4.522\pm0.778$ & $2.797\pm0.235$ & $2.010\pm0.468$ & $0.594\pm0.013$ & $0.624\pm0.033$ \\
     & 1D-CNN      & $1.790\pm0.184$ & $3.390\pm1.359$ & $2.222\pm0.247$ & $1.633\pm0.324$ & $0.508\pm0.018$ & $0.527\pm0.039$ \\
     & Transformer & $2.108\pm0.321$ & $4.629\pm1.278$ & $2.312\pm0.165$ & $1.818\pm0.315$ & $0.551\pm0.017$ & $0.580\pm0.045$ \\
     & \textbf{PALM} & $\mathbf{1.263\pm0.153}$ & $\mathbf{2.585\pm0.991}$ & $\mathbf{1.546\pm0.131}$ & $\mathbf{1.565\pm0.214}$ & $\mathbf{0.494\pm0.012}$ & $\mathbf{0.457\pm0.021}$ \\
\midrule
168h & Mamba2      & $2.588\pm0.290$ & $4.848\pm0.863$ & $2.922\pm0.215$ & $2.021\pm0.336$ & $0.596\pm0.010$ & $0.627\pm0.009$ \\
     & 1D-CNN      & $2.009\pm0.301$ & $3.748\pm1.012$ & $2.516\pm0.299$ & $1.667\pm0.346$ & $0.523\pm0.008$ & $0.556\pm0.015$ \\
     & Transformer & $2.159\pm0.341$ & $4.724\pm1.126$ & $2.294\pm0.198$ & $1.747\pm0.312$ & $0.537\pm0.011$ & $0.580\pm0.025$ \\
     & \textbf{PALM} & $\mathbf{1.334\pm0.167}$ & $\mathbf{2.849\pm1.069}$ & $\mathbf{1.638\pm0.139}$ & $\mathbf{1.610\pm0.215}$ & $\mathbf{0.506\pm0.012}$ & $\mathbf{0.476\pm0.019}$ \\
\bottomrule
\end{tabular}
\end{table}

\paragraph{Comparison to zero-shot forecasting foundation models.}
Beyond the parameter-matched trained baselines, we compare against three pretrained time-series forecasting foundation models --- Chronos-Bolt~\citep{ansari2024chronos}, TimesFM-2.5~\citep{abhimanyu2024decoder}, and Moirai~\citep{woo2024unified} --- applied zero-shot (no weather fine-tuning) under a matched 48\,h context, the same spatial-OOD split, and the same metric. On normalized-space MSE, PALM has the lowest error at every horizon and is also below the trained baselines. At the shortest horizon, PALM's error is 12.5\% lower than that of the best foundation models (Chronos-Bolt and TimesFM-2.5). PALM is trained on the target split, and the foundation models are applied without weather adaptation, so the comparison contrasts task-trained anchoring with zero-shot transfer rather than two matched training regimes.

\subsection{Mechanism: Anchored Memory Timescales}
\label{sec:results:mechanism}

We measure directly whether the anchoring places memory at the signal's physical timescales, first through the learned parameters and then through the model's input--output behavior.

Proposition~\ref{prop:gauge} makes $\Delta|A|$ the only identifiable quantity, and the dual anchoring gives the latent timescale $\tau=1/|A|$ physical units (TR). In the full model the median $\tau$ spans several volumes. Without the anchoring (freely learned or random $\Delta$, no physical $A$-initialization), $\tau$ is not referenced to acquisition units, and its median in the same coordinates is about 3\% of the anchored value (Appendix~\ref{sec:appx:hcp_tr_shift}). Representation stability under context truncation follows the same order: CKA is highest for the full model, then learnable $\Delta$, then random $\Delta$ (Table~\ref{tab:drift2tr}).

We then estimate each model's induced memory kernel (the input--output Jacobian profile) and compare it with the spectrum-implied target. On Weather-5K, PALM's kernel is 39\% closer than Mamba2's to the kernel implied by the same spectral prior and 25\% closer than the Transformer's, with the same ordering on HCP. The fixed-kernel 1D-CNN is lower still, and Appendix~\ref{sec:appx:induced_kernel} gives the estimator, its validation, and the full numbers.

\section{Discussion}
\label{sec:discussion}

Scientific foundation models are built to reuse representations across acquisition conditions, yet their representations degrade under distribution shift. We traced one source of this failure to selective SSM backbones. In-distribution data leave the split between $\Delta$ and $A$ unreferenced, so the model's memory no longer matches the signal's timescales once the acquisition condition changes. PALM fixes this gauge by anchoring both factors to acquisition units, and its latent timescales stay in physical units. Across fMRI and weather forecasting, PALM keeps representations more stable under shift and forecasts unseen stations more accurately than parameter-matched backbones. Representation analyses support these results. PALM's induced memory kernel lies closer to the spectrum-implied kernel, and its representations stay more stable than those of a fixed multi-scale grid.

The central step is to disentangle $\Delta$ and $A$ and return each to its own physical role. $\Delta$ sets the discretization step, and PALM calibrates it to the acquisition unit through the signal's spectrum. $A$ sets how long the model remembers, and PALM ties it to a physical memory range. Once separated, latent timescales become interpretable and comparable across conditions, and $\Delta$ is no longer a free hyperparameter mistaken for a sampling interval. Physically anchored timescales are therefore a design principle for scientific foundation models: reuse under shift depends on grounding the backbone's memory in how the signal is physically sampled, not only on more data or capacity.

\newpage
\subsection*{AI use statement}

Generative AI tools (Claude) were used for code assistance during implementation, to collate experimental results, and for drafting portions of this manuscript. All AI-generated content was reviewed, verified, and edited by the authors. We take full responsibility for the final content.

\subsection*{Ethics statement}

This work uses publicly available, de-identified human neuroimaging data (Human Connectome Project) under their respective data use agreements, together with the public Weather-5K dataset. As with any model applied to human neuroimaging, decoded representations warrant attention to consent, privacy, and cross-population generalization.

\subsection*{Reproducibility statement}

Training hyperparameters are given in Appendix~\ref{sec:appx:training_details}, architecture details in Appendix~\ref{sec:appx:arch_details}, dataset and preprocessing details in Appendix~\ref{sec:appx:data}, proofs of the theoretical results in Appendix~\ref{sec:appx:proofs}, and a proof sketch of the approximation result in Appendix~\ref{sec:appx:kernel_approx}.

\bibliography{references}
\bibliographystyle{iclr2027_conference}

\newpage
\appendix

\section{Extended Related Work}
\label{sec:appx:related}

\paragraph{Foundation models for scientific time series.}
Brain foundation models and time-series forecasting foundation models demonstrate the value of large-scale pretraining, zero-/low-shot adaptation, and reusable representations across datasets~\citep{tak2026generalizable,wang2025towards,wang2025slim,wang2026omni,abhimanyu2024decoder,ansari2024chronos,woo2024unified}.
These works primarily ask how far representation learning can be pushed by larger corpora, generic sequence objectives, or adaptation modules.
PALM is complementary: it asks whether the backbone itself preserves physically meaningful temporal scales under shift, even without modality-specific components or shift-specific supervision.

\paragraph{Temporal scale priors in sequence models.}
The prior that dynamics are organized across a (roughly logarithmic) ladder of timescales is well established, and PALM builds on it. Scale-space theory formalizes the logarithmic organization of scale~\citep{lindeberg1994scale}, and time-rescaling-invariant convolutional models make log-spaced timescales explicit~\citep{jacques2022time}; in neuroscience, hierarchical temporal receptive windows describe multi-timescale processing in cortex~\citep{hasson2015hierarchical}. In sequence models, Tallec and Ollivier show that timescales can be injected through gate initialization~\citep{tallec2018can}, and HiPPO/S4-style analyses give $\Delta$ a timescale interpretation while treating its range as a tuned hyperparameter~\citep{gu2020hippo}. PALM builds on this lineage and addresses \emph{how to encode multiple timescales in the discretization of a selective SSM so that training keeps their physical grounding}. Concretely, a fixed log grid encodes rescaling \emph{invariance} (the correct choice when time-rescaling is arbitrary), whereas PALM performs measured \emph{alignment} plus gauge fixing (the correct choice when the acquisition unit is known and the spectrum is piecewise); we test this invariance-vs-alignment contrast directly as a baseline (Section~\ref{sec:results:hcp_tr_shift}).

\paragraph{Multi-scale state-space and Mamba models.}
Recent multi-scale SSM/Mamba variants motivate the need for temporal hierarchy, but instantiate scale in different ways.
MS-SSM~\citep{Karami2025MSSSMAM} processes multiple resolutions with specialized state-space dynamics and an input-dependent scale mixer.
ms-Mamba~\citep{karadag2025ms} uses multiple Mamba blocks with different sampling rates to model forecasting signals at several temporal scales.
STM3~\citep{chen2025stm3} targets long-term spatio-temporal prediction with multiscale Mamba experts, adaptive graph causal convolution, and mixture-of-experts routing.
WMF-Traffic~\citep{li2025multi} combines wavelet decomposition, traffic-aware Mamba, and Fourier adjustment for traffic forecasting.
The key structural distinction is identifiability: these approaches set multiple sampling rates or discretization values but leave $\Delta$ entangled with $A$, so by Proposition~\ref{prop:gauge} the physical timescale remains a free direction of training. PALM instead derives scale-specific discretization from empirical spectra and, crucially, anchors \emph{both} $\Delta$ (to acquisition units at the knees) and $A$ (to memory timescales), controlling the $A$--$\Delta$ coupling so the physical grounding survives shift. Our comparisons use parameter-matched, widely established backbones (Mamba2, Transformer, 1D-CNN), a fixed log-spaced scale-space grid that isolates multi-scale capacity inside the same backbone, and zero-shot forecasting foundation models. We position the concurrent variants in the design space by how they treat $\Delta$ and $A$. Several of them lack mature public implementations or target different benchmarks.

\section{Formal Problem Setup}
\label{sec:appx:problem_setup}

We observe a multivariate time series $x_{1:T} \in \mathbb{R}^{d}$.
A temporal encoder $\Phi_{\theta}$ maps $x_{1:T}$ to latent states $h_{1:T} = \Phi_{\theta}(x_{1:T})$, $h_t \in \mathbb{R}^m$.
An aggregation $z = \rho(h_{1:T})$ feeds a prediction head $\hat{y} = g_{\psi}(z)$.
Given labeled source data $\mathcal{D}_{e}$, training minimizes
\begin{equation}
\min_{\theta,\psi}\;\mathbb{E}_{(x,y)\sim \mathcal{D}_{e}}\!\left[\ell\!\left(g_{\psi}(\rho(\Phi_{\theta}(x))),\,y\right)\right].
\end{equation}
Let $e \in \mathcal{E}$ index an acquisition or domain condition (fMRI TR, brain state, scanner/site, meteorological station regime). A structured shift $e \to e'$ alters the temporal statistics of $x_{1:T}$ while preserving the downstream task definition. Representation drift under shift $e \to e'$ is defined in Equation~\ref{eq:drift_def} below.
The physics-informed spectral assumption is: the power spectral density of $x$ admits a piecewise scaling form $S_x(f) \approx c_k f^{-\beta_k}$, $f \in [f_{k-1}, f_k]$, $k=1,\dots,K$, motivating scale-specific discretization parameters $\Delta^{(k)}$ tied to the knee frequencies $(f_1,\dots,f_{K-1})$.

The central object in this setting is the temporal weighting profile induced by the encoder.
Given a multivariate time series $x_{1:T}$, a temporal encoder $\Phi_{\theta}$ induces condition-dependent latent states
\begin{equation}
h_{1:T}^{(e)}=\Phi_{\theta}^{\tilde g_e}(x_{1:T}),
\qquad
z_e(x)=\rho(h_{1:T}^{(e)}),
\label{eq:encoder_kernel_setup}
\end{equation}
where $\tilde g_e$ denotes the model-induced temporal weighting profile under condition $e$, and $\rho$ aggregates latent states into a representation.
A task head $g_{\psi}$ predicts $\hat y=g_{\psi}(z_e(x))$.
Under shift, the question is whether the architecture preserves a useful temporal weighting profile and representation geometry when the condition changes to $e'$, without shift-specific retraining or adaptation.
We quantify representation drift between conditions via a geometry-aware discrepancy $\mathcal{D}$, instantiated with CKA~\citep{kornblith2019similarity} or distance correlation dCor~\citep{szekely2007measuring}:
\begin{equation}
\mathrm{Drift}(e,e') := \mathbb{E}_{x}\!\left[\mathcal{D}\!\left(z_e(x),\, z_{e'}(x)\right)\right],
\label{eq:drift_def}
\end{equation}
and report both downstream performance and $\mathrm{Drift}(e,e')$ for each shift pair.
For Weather-5K, the same setup is used with a regression head and variable-wise MAE; for HCP, $g_{\psi}$ is a classifier and $z$ is evaluated through accuracy and representation-stability metrics.

\section{Theory Proofs}
\label{sec:appx:proofs}

\subsection{Memory as Kernel Matching: Setup and Statement}
\label{sec:appx:kernelmatch_setup}

Viewing the encoder locally as a convolution, the model's induced kernel $\tilde g$ approximates the true kernel $g$ of a given condition, and the resulting representation error is controlled by their $L^1$ mismatch: $\|h-\tilde h\|_\infty \le M\,\|g-\tilde g\|_1$ for bounded inputs $\|x\|_\infty\le M$ (Lemma~\ref{lem:kernel_mismatch}, the $p{=}1,q{=}\infty$ case of Young's inequality, stated as a bridge rather than a contribution). For a linear readout this in turn bounds prediction instability (Proposition~\ref{prop:readout_sensitivity}), so memory-timescale misalignment propagates to downstream error under shift. Multi-scale kernels enlarge the approximating family in a structured way---a nonnegative mixture of ordered exponentials whose timescales are tied to spectrum-derived knees (Equation~\ref{eq:mixture_exp_kernel})---lowering the achievable mismatch $\inf_{\tilde g\in\mathcal{G}_K}\|g-\tilde g\|_1$ without shift-specific adaptation. A power-law approximation view is given in Appendix~\ref{sec:appx:kernel_approx}.

Many SSM-based encoders admit a locally linear convolutional view~\citep{gu2021efficiently,gu2023mamba}:
\begin{equation}
h(t) \;=\; \int_{0}^{\infty} g(\tau)\, x(t-\tau)\, d\tau,
\qquad
\tilde h(t) \;=\; \int_{0}^{\infty} \tilde g(\tau)\, x(t-\tau)\, d\tau,
\label{eq:kernel_conv}
\end{equation}
where $g$ is the \emph{true} effective kernel under a given condition (or the kernel required to preserve task-relevant dynamics), and $\tilde g$ is the kernel induced/approximated by the model.

\begin{lemma}[Kernel mismatch bound]
\label{lem:kernel_mismatch}
Assume $\|x\|_{\infty}\le M$ (here $M$ is a bound on the input amplitude, with the same units as $x$).
Let $h$ and $\tilde h$ be defined as in Equation~\ref{eq:kernel_conv}.
Then,
\begin{equation}
\|h - \tilde h\|_{\infty} \;\le\; M\, \|g - \tilde g\|_{1}.
\label{eq:kernel_mismatch_bound}
\end{equation}
\end{lemma}
Lemma~\ref{lem:kernel_mismatch} is the $p=1,\,q=\infty$ special case of Young's convolution inequality; we state it as a bridge that makes the chain from kernel error to representation error to prediction error precise, not as a contribution in itself. For a linear readout, representation drift also bounds prediction instability (Proposition~\ref{prop:readout_sensitivity}), giving $\|\Delta h\| \lesssim M\,\|g_{e'} - \tilde g\|_{1}$, so memory-timescale misalignment can propagate to downstream error under shift. Multi-scale kernels enlarge the approximating family in a structured way:
\begin{equation}
g_K(t) \;=\; \sum_{k=1}^{K} \alpha_k\, e^{-t/\tau_k},
\qquad \alpha_k\ge 0,\ \sum_k \alpha_k=1,\ \tau_1 \ge \tau_2 \ge \cdots \ge \tau_K,
\label{eq:mixture_exp_kernel}
\end{equation}
and by tying $\{\tau_k\}$ to spectrum-derived knees with an ordered hierarchy PALM reduces the achievable mismatch $\inf_{\tilde g \in \mathcal{G}_K}\|g_{e'}-\tilde g\|_1$ without relying on shift-specific adaptation.

\subsection{Why the Multi-Scale Structure Lives on \texorpdfstring{$\Delta$}{Delta}, and \texorpdfstring{$A$}{A} is Anchored to a Single Range}
\label{sec:appx:where_multiscale}
Since the kernel sees $\Delta$ and $A$ only through the product $\Delta|A|$ (Proposition~\ref{prop:gauge}), placing a $K$-fold ladder on $\Delta$ or on $A$ yields the same in-distribution kernel family; which factor should \emph{carry} the multi-scale structure is therefore decided by identifiability, not by the I/O map. The asymmetry is that only $\Delta$ has an external, measurable target: the ladder locations are the spectrum knees $(f_1,\dots,f_{K-1})$, which live in acquisition units and can be estimated from the PSD. The conjugate factor $\tau = 1/|A|$ is exactly the flat direction --- there is no data-derived target for a \emph{per-scale} $A$, so a free scale-structured $A$-ladder would place scales at unmeasurable locations, reintroducing the assumed-grid choice PALM is designed to avoid. We therefore put the measured multi-scale structure on $\Delta$ (alignment) and keep $A$ as a single physical range, drawn per head so within-scale memory still varies but the scale separation is not attributed to $A$. Crucially, a free per-scale $A$ would also re-open the gauge freedom \emph{within each scale} ($(\Delta_k,A_k)\!\mapsto\!(c_k\Delta_k, c_k^{-1}A_k)$), letting each $A_k$ silently absorb its $\Delta_k$ and rendering the spectrum anchoring cosmetic; a single anchored $A$ is the minimal constraint that keeps the $\Delta$-anchoring load-bearing.

\subsection{Proof of Lemma~\ref{lem:kernel_mismatch} (Kernel Mismatch Bound)}

\begin{proof}
For any $t$,
\[
|h(t)-\tilde h(t)|
=
\left|\int_{0}^{\infty} (g(\tau)-\tilde g(\tau))\, x(t-\tau)\, d\tau\right|
\le
\int_{0}^{\infty} |g(\tau)-\tilde g(\tau)|\, |x(t-\tau)|\, d\tau
\le
M \int_{0}^{\infty} |g(\tau)-\tilde g(\tau)|\, d\tau.
\]
Taking the supremum over $t$ gives Equation~\ref{eq:kernel_mismatch_bound}.
\end{proof}

\subsection{Proposition: Readout Sensitivity to Representation Drift}

\begin{proposition}[Readout sensitivity]
\label{prop:readout_sensitivity}
For any $w$ and representations $(h,h')$,
\begin{equation}
|\hat y' - \hat y| = |w^\top(h'-h)| \le \|w\|_2\,\|h'-h\|_2.
\end{equation}
\end{proposition}

\begin{proof}
Cauchy--Schwarz inequality applied directly.
\end{proof}

\section{Multi-scale Kernels Reduce Approximation Error to Power-Law Targets}
\label{sec:appx:kernel_approx}

We provide a theoretical result that formalizes the advantage of multi-scale exponential mixtures over single-scale models when approximating power-law decay kernels, which arise naturally from the piecewise power-law spectral structure of neural signals.

\paragraph{Setup.}
Consider the power-law decay kernel
\begin{equation}
g(t) = (1+t)^{-\alpha}, \qquad \alpha > 1,
\label{eq:powerlaw_kernel}
\end{equation}
which corresponds to long-memory temporal dependence.
This arises when the effective temporal kernel of a system integrates contributions from a continuum of timescales.
We analyze approximation over a \emph{finite dynamic range} $[t_{\min}, t_{\max}]$ with $0 < t_{\min} < t_{\max} < \infty$. This is not a technical convenience but the physically meaningful setting: the relevant lags run from the acquisition interval (TR) to the run length, which is exactly the range that acquisition anchoring (Section~\ref{sec:methods}) makes explicit. Restricting to this range also lets us invoke exponential-sum approximation theory, which gives a rate far stronger than the algebraic one and avoids the ill-defined $L^1[0,\infty)$ conversion of a naive quadrature argument.
Let $\mathcal{G}_K$ denote the family of mixtures of $K$ exponentials:
\begin{equation}
\mathcal{G}_K \;=\; \left\{ \tilde{g}(t) = \sum_{k=1}^{K} a_k\, e^{-\lambda_k t} \;:\; a_k \ge 0,\; \sum_{k=1}^K a_k = 1,\; \lambda_k > 0 \right\}.
\label{eq:G_K}
\end{equation}
An SSM whose effective dynamics are dominated by one discretization scale can be viewed as a $K=1$ approximation, while PALM with $K$ explicit scales corresponds to $\mathcal{G}_K$. This contrast is between $K$ scales and a single scale, not between PALM and Mamba-2, whose heads carry distinct $A$ values.

\begin{proposition}[Finite-range multi-scale approximation of power-law kernels]
\label{prop:kernel_approx}
Let $g(t) = (1+t)^{-\alpha}$ with $\alpha > 1$, and fix a dynamic range $[t_{\min}, t_{\max}]$. There exist constants $C_\alpha, c > 0$ (depending on $\alpha$) such that the best $K$-term exponential sum satisfies
\begin{equation}
\inf_{\tilde{g} \in \mathcal{G}_K} \|g - \tilde{g}\|_{L^1[t_{\min}, t_{\max}]} \;\leq\; C_\alpha \, (t_{\max}-t_{\min}) \, \exp\!\left(-\,\frac{c\,K}{\log(t_{\max}/t_{\min})}\right).
\label{eq:multi_scale_approx_bound}
\end{equation}
In contrast, a single exponential ($K=1$) cannot track a multi-decade algebraic decay: on the full half-line its error is bounded below,
\begin{equation}
\inf_{\tilde{g} \in \mathcal{G}_1} \|g - \tilde{g}\|_{L^1[0,\infty)} \;=\; \Omega(1),
\label{eq:single_scale_floor}
\end{equation}
independent of the choice of $\lambda_1$.
\end{proposition}

\paragraph{Proof sketch.}
The key insight is the Stieltjes (Laplace) representation of $g$:
\begin{equation}
g(t) = (1+t)^{-\alpha} = \frac{1}{\Gamma(\alpha)} \int_0^\infty \lambda^{\alpha-1} e^{-\lambda} e^{-\lambda t}\, d\lambda \;\propto\; \int_0^\infty e^{-\lambda t}\, \mu(\lambda)\, d\lambda,
\label{eq:stieltjes}
\end{equation}
where the mixing density is $\mu(\lambda) = \lambda^{\alpha-1}e^{-\lambda}/\Gamma(\alpha)$ for $\lambda \in (0, \infty)$.
Thus $g$ is a continuous mixture of exponentials, and the best $K$-exponential approximation corresponds to a $K$-point quadrature rule for the measure $\mu$.

\emph{Upper bound (Equation~\ref{eq:multi_scale_approx_bound}).}
The Stieltjes representation in Equation~\ref{eq:stieltjes} exhibits $g$ as a completely monotone function on $[t_{\min}, t_{\max}]$. For such functions, the theory of best exponential-sum (Prony/sinc-quadrature) approximation on a finite interval gives an error that decays \emph{exponentially} in the number of terms, with the exponent scaling inversely in the logarithm of the dynamic range $t_{\max}/t_{\min}$~\citep{braess2005approximation,beylkin2005approximation}. Applying this bound on $[t_{\min}, t_{\max}]$ and multiplying the uniform error by the interval length to pass to $L^1$ yields Equation~\ref{eq:multi_scale_approx_bound}. This finite-range rate is strictly stronger than the algebraic $K^{-(\alpha-1)}$ rate obtainable on $[0,\infty)$, and it is available precisely because acquisition anchoring fixes the physical range over which the kernel must be matched.

\emph{Lower bound (Equation~\ref{eq:single_scale_floor}).}
A single exponential $a_1 e^{-\lambda_1 t}$ cannot reproduce the algebraic tail of $(1+t)^{-\alpha}$: for any $\lambda_1 > 0$, the $L^1$ difference $\int_0^\infty |(1+t)^{-\alpha} - a_1 e^{-\lambda_1 t}|\, dt$ remains bounded away from zero because exponentials decay faster than any power law.
This follows from the fact that $e^{-\lambda t} = o((1+t)^{-\alpha})$ as $t \to \infty$ for all $\lambda > 0$ and $\alpha > 1$, giving an $\Omega(1)$ tail discrepancy regardless of the choice of $a_1, \lambda_1$.

\paragraph{Implication for PALM.}
Proposition~\ref{prop:kernel_approx} establishes that, over the physically relevant dynamic range, the approximation error of $\mathcal{G}_K$ to a power-law kernel decreases \emph{exponentially} in the number of scales $K$, while a single exponential ($K=1$) cannot track the algebraic decay and incurs a floor error. A small number of scales therefore already captures most of the achievable reduction, consistent with the empirical finding that $K=3$ is a safe default (Appendix~\ref{sec:appx:k_sensitivity}); we thus make the rate qualitative rather than assert a specific numerical factor.
Combined with Lemma~\ref{lem:kernel_mismatch} (Section~\ref{sec:theory:kernelbound}), this suggests why a small-$K$ multi-scale SSM can reduce representation error relative to a one-scale approximation when the target kernel is well modeled by a power-law decay --- provided, by Proposition~\ref{prop:gauge}, that the scales are anchored to physical units rather than left to slide along the gauge orbit.

\paragraph{Physics-informed quadrature.}
Proposition~\ref{prop:kernel_approx} bounds only the best $K$-term approximation, and a particular quadrature, such as knee-placed nodes, need not attain it.
PALM's SpectralHyperNet provides a structured quadrature: the knee frequencies $(f_1, f_2)$ inferred from the empirical PSD determine the quadrature nodes $\{\lambda_k\}$ in a data-adaptive manner.
Compared to randomly initialized or freely learnable scale parameters (which may cluster in suboptimal regions of the $\lambda$ axis), the physics-informed initialization places the quadrature nodes near the spectral knees of the target distribution, yielding a better-conditioned approximation.
This provides a theoretical mechanism underlying the empirical advantage of SpectralHyperNet-derived $\Delta$ parameters over random or learnable alternatives in the TR-shift experiments (Appendix~\ref{sec:appx:hcp_tr_shift}).

\section{Dataset and Preprocessing Details}
\label{sec:appx:data}

\paragraph{HCP-1200 motor task fMRI.}
Time series were extracted from 360 cortical regions (HCPMMP1 atlas~\citep{glasser2016multi}), yielding inputs $\mathbf{x}_{1:T}\in\mathbb{R}^{T\times360}$ (TR$=0.72$\,s).
Event-related segments for five motor conditions (LH, RH, LF, RF, T) were extracted with a maximum block length of 17 TRs; each ROI time series was z-score normalized per run.
\textbf{Temporal-context shift:} early-window models are trained and evaluated with SSM backbone input restricted to the first 3 or 2 TRs of each motor block, while SpectralHyperNet receives the full 284-TR motor run for spectral calibration. This tests early-window decoding with full-run spectral scale estimates, not a change in scanner TR or sampling rate.
\textbf{Low-resource:} labeled fraction varied in $\{1\%, 5\%, 10\%, 100\%\}$; validation and test sets unchanged.

\paragraph{HCP resting-state fMRI (brain-state shift).}
Axis (iii) follows the modality-shift loader used by the HCP cross-dataset script. Subjects are restricted to participants with both resting-state and motor-task files, then split randomly into train/validation/test subject sets. Train and validation loaders use resting-state fMRI, while the held-out test loader uses HCPMMP1 motor-task fMRI (MOTOR\_RL); all sequences are padded or truncated to 284 TRs and z-score normalized. Models are trained from scratch for sex classification using the Gender metadata label, which is shared across scan conditions, and evaluated on the task-fMRI test split without adaptation.

\paragraph{Weather-5K.}
Large-scale meteorological benchmark~\citep{han2024far} spanning 5,672 global weather stations.
We use hourly station records with five raw meteorological variables: temperature (TMP), dew point (DEW), wind angle, wind rate, and sea-level pressure (SLP). Wind angle is circular, so it is represented as sine and cosine components, giving six model channels: TMP, DEW, wind-sin, wind-cos, wind rate, and SLP. Missing rows are dropped before window construction.

For matched experiments we use a 100-station geo-stratified subset. Stations are sampled from geographic clusters built from latitude, longitude, elevation, and station quality metadata, with a target cluster size of eight stations. The spatial OOD split then partitions stations, not windows: stations are clustered on unit-sphere geographic coordinates, clusters are ranked by geographic isolation, and the most isolated cluster(s) are assigned to the held-out test set. We use this station-disjoint split as the primary Weather-5K evaluation because chronological holdouts share station identity between train and test, while spatial holdout evaluates generalization to unseen station domains. In the current runs, 90\% of stations are used for training/validation and 10\% are held out for spatial OOD testing. Within the train-station set, $1/9$ of windows are used for validation, yielding an approximate 8:1:1 train/validation/test station-window budget while keeping test stations disjoint from training.

Normalization statistics for z-scoring are computed from training inputs only. For train stations, per-station mean and standard deviation are estimated from training windows; held-out spatial OOD stations fall back to the global training statistics because their station-specific statistics are unavailable during training. RevIN is then applied on top of this dataset-level normalization during model training and evaluation, and predictions are mapped back to variable units before reporting MAE.

\section{Sensitivity to Scale Count and Knee Estimation}
\label{sec:appx:k_sensitivity}

We report robustness to the two design choices most likely to concern a reader: the number of scales $K$ and the accuracy of the estimated knees (2\,TR accuracy, five seeds, consistent seed set).

\paragraph{Scale count $K$.}
Varying $K\in\{1,2,3,4,5\}$ gives 2\,TR accuracy $0.925 / 0.929 / 0.929 / 0.931 / 0.927$ --- flat to $\sim$0.6\,pp.\footnote{This sweep varies only $K$, so the parameter count changes slightly with it (from $\approx$6.17M at $K=1$ to $\approx$5.91M at $K=4$, a $\approx$0.26M spread around the 5.98M default); we do not re-match parameters across $K$ because the conclusion is that accuracy is \emph{flat}, not that a particular $K$ is optimal, and the spread runs opposite to any capacity trend.} Thus $K=3$ is a safe default rather than a sharply tuned optimum. The plateau is expected on this short, narrow-separation window, where the discriminative signal is low-dimensional and the hierarchy is carried mainly by the $A$-timescales (Appendix~\ref{sec:appx:induced_kernel}, Section~\ref{sec:results:hcp_tr_shift}); it does \emph{not} imply scale count is irrelevant in general. On Weather-5K, where spectral separation is roughly an order of magnitude wider, scale count is genuinely load-bearing.

\paragraph{Knee-estimation error.}
Rescaling the training-set knees by a factor in $[0.5, 1.5]$ moves 2\,TR accuracy by $<0.6$\,pp over the whole range. We read this as robustness to moderate knee-estimation error rather than a sharp peak at the ``true'' knees, consistent with the finite-range approximation view (Appendix~\ref{sec:appx:kernel_approx}): a small number of well-placed scales already captures most of the achievable kernel match, and moderate misplacement is absorbed by the trainable weights.

\paragraph{Departure from a canonical log grid.}
The scales SpectralHyperNet infers are log-like in magnitude but measurably not an exact geometric ladder: across five seeds the inferred pre-discretization placement has ratio-of-ratios $0.526\pm0.003$ (a value of $1.000$ would be an exact log grid), a reproducible $1.9\times$ departure. This is the quantitative content of ``measured alignment, not rescaling invariance'': scale-space theory prescribes log spacing but not \emph{where} the knees sit, which is what the spectral module supplies.

\section{Piecewise Power-Law Spectra Across Datasets}
\label{sec:appx:psd_multifractal}

\begin{figure}[htbp]
\centering
\begin{minipage}[b]{0.49\linewidth}
  \centering
  \includegraphics[width=\linewidth]{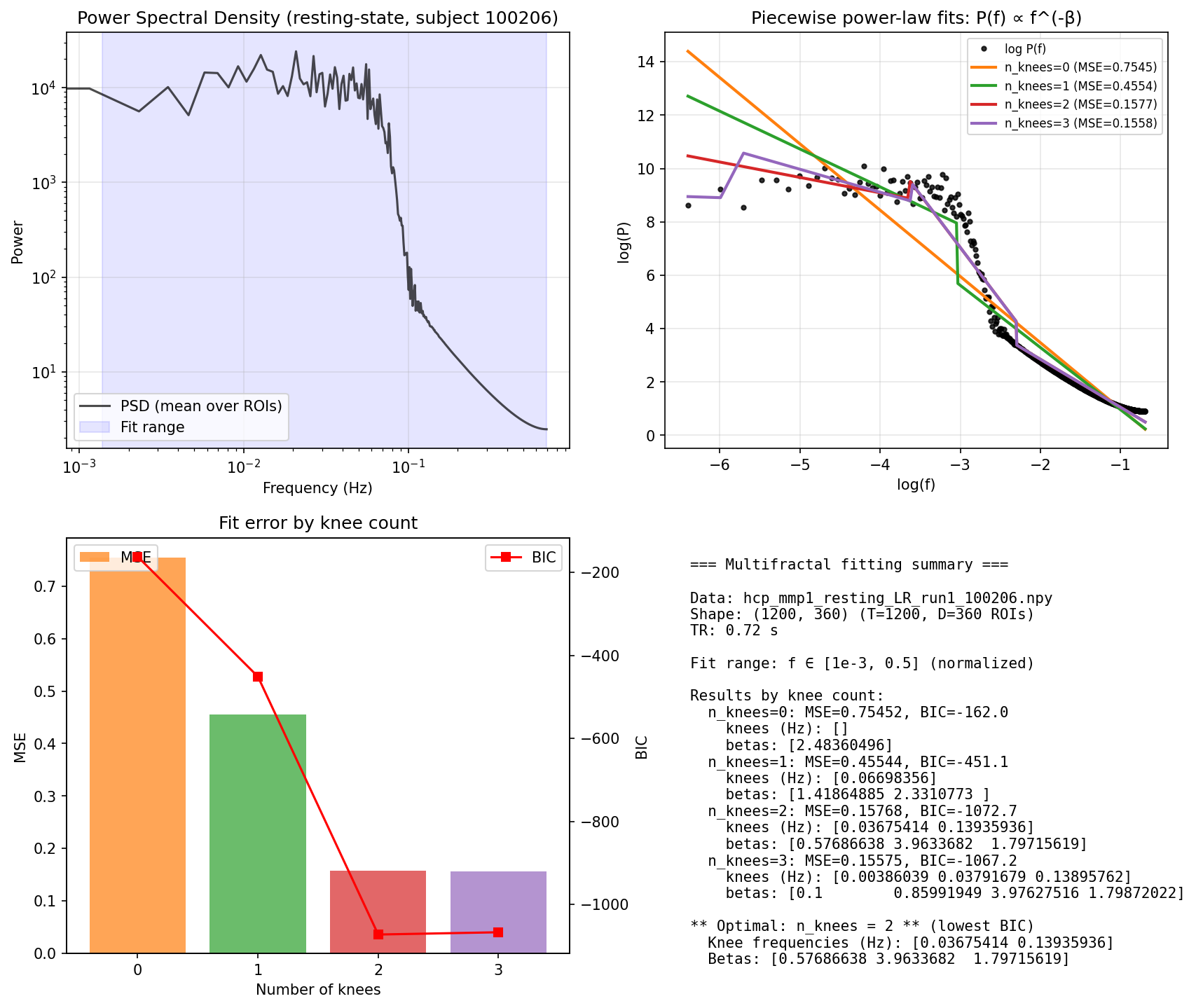}
\end{minipage}
\hfill
\begin{minipage}[b]{0.49\linewidth}
  \centering
  \includegraphics[width=\linewidth]{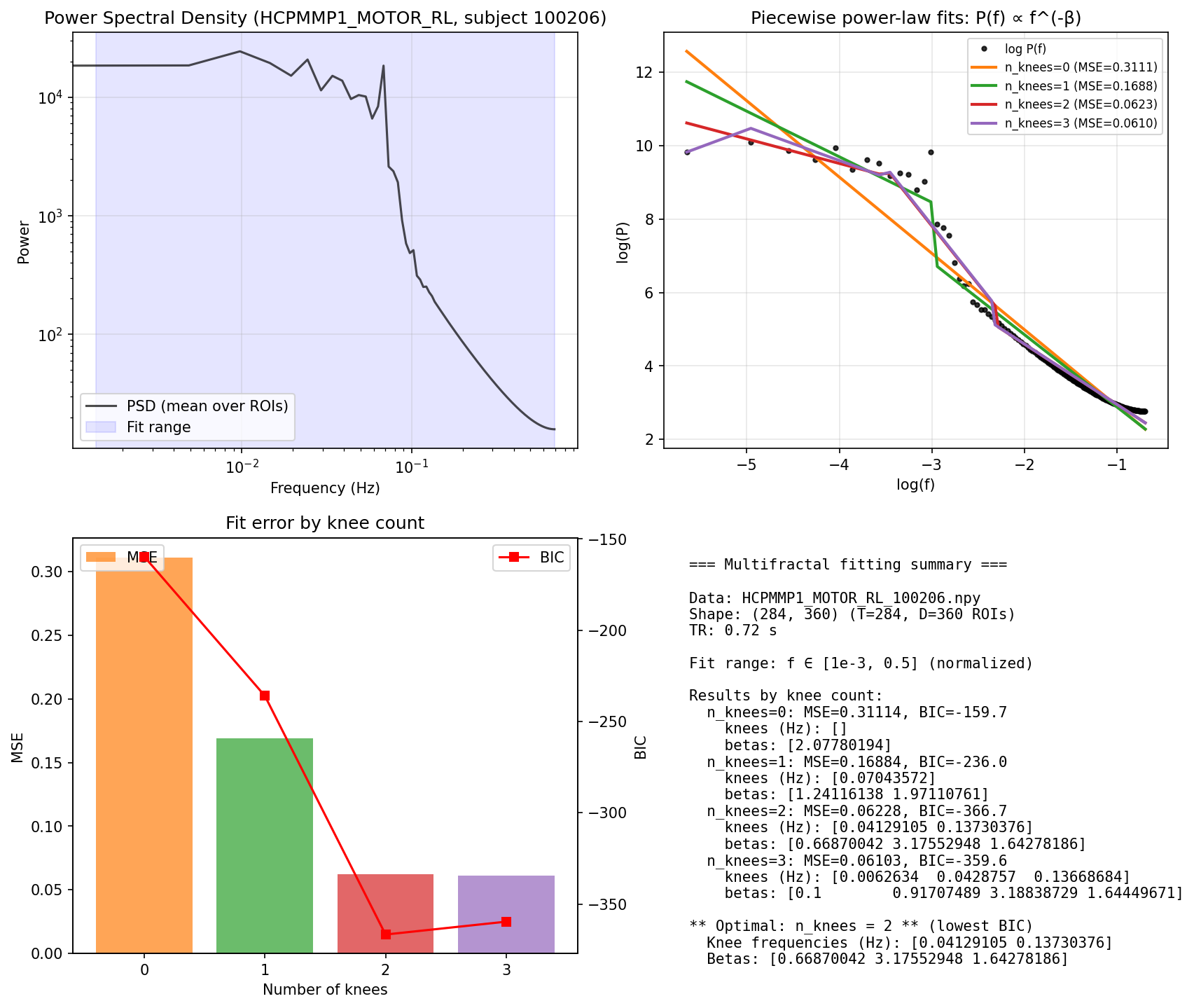}
\end{minipage}

\vspace{0.5em}

\begin{minipage}[b]{0.58\linewidth}
  \centering
  \includegraphics[width=\linewidth]{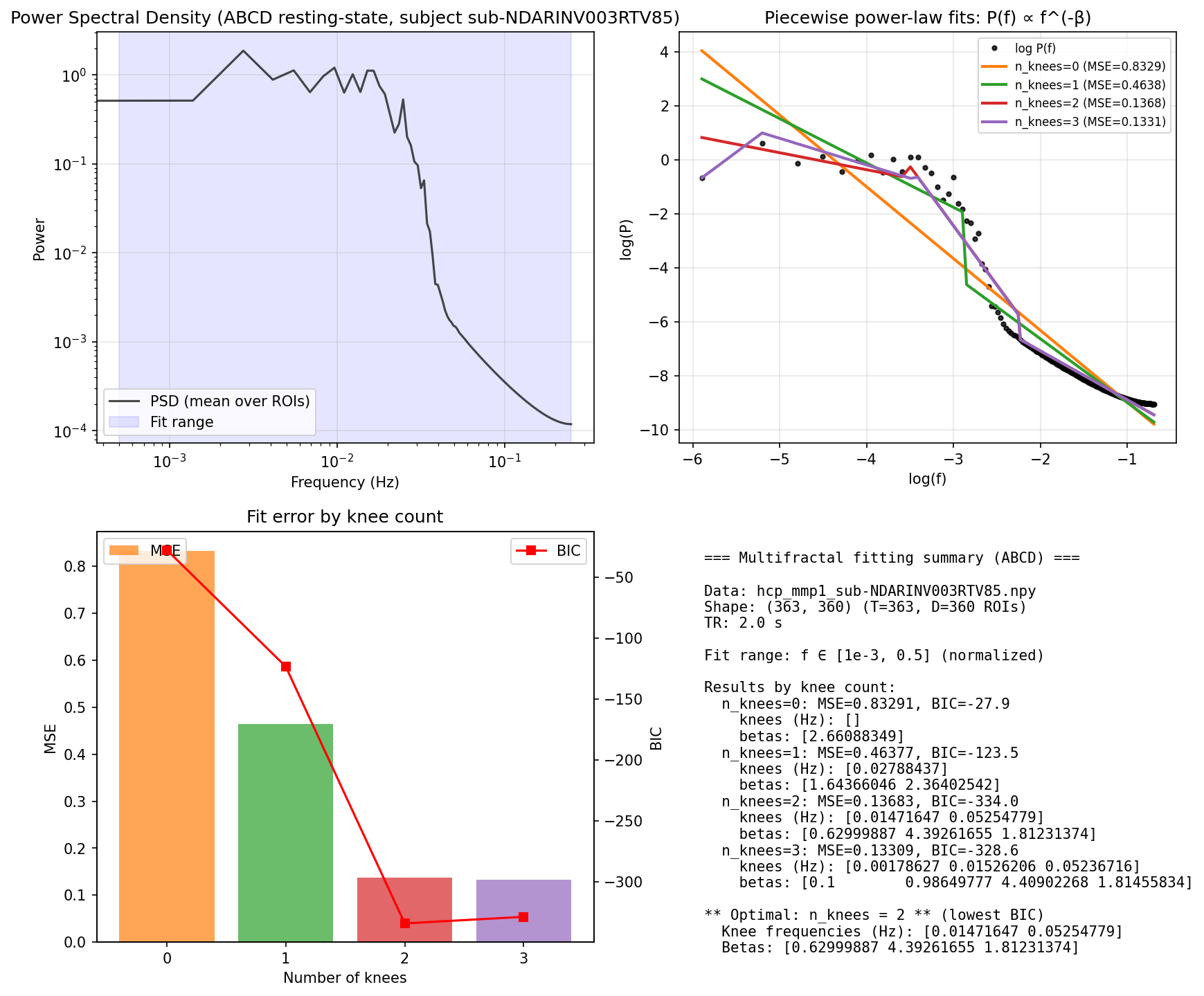}
\end{minipage}

\caption{
\textbf{Piecewise power-law structure of empirical PSD across neural datasets.}
\textbf{(top left)} HCP resting-state fMRI (TR$=0.72$\,s, 360 cortical ROIs). The BOLD PSD exhibits distinct scaling regimes separated by knee frequencies $(f_1, f_2)$, consistent with the piecewise power-law model assumed by SpectralHyperNet.
\textbf{(top right)} HCP motor task fMRI (TR$=0.72$\,s, 360 ROIs). Task-evoked BOLD signals display a similar piecewise structure, indicating that the multi-scale power-law structure is not restricted to resting-state dynamics.
\textbf{(bottom)} ABCD resting-state fMRI. The piecewise power-law structure is preserved across a large developmental cohort with different acquisition parameters.
In each panel, the fitted piecewise model (dashed lines) accurately captures the observed spectral slope changes at the estimated knee frequencies, supporting the core spectral modeling assumption of PALM.
}
\label{fig:appx:psd_neural}
\end{figure}

\begin{figure}[htbp]
\centering
\begin{minipage}[b]{0.49\linewidth}
  \centering
  \includegraphics[width=\linewidth]{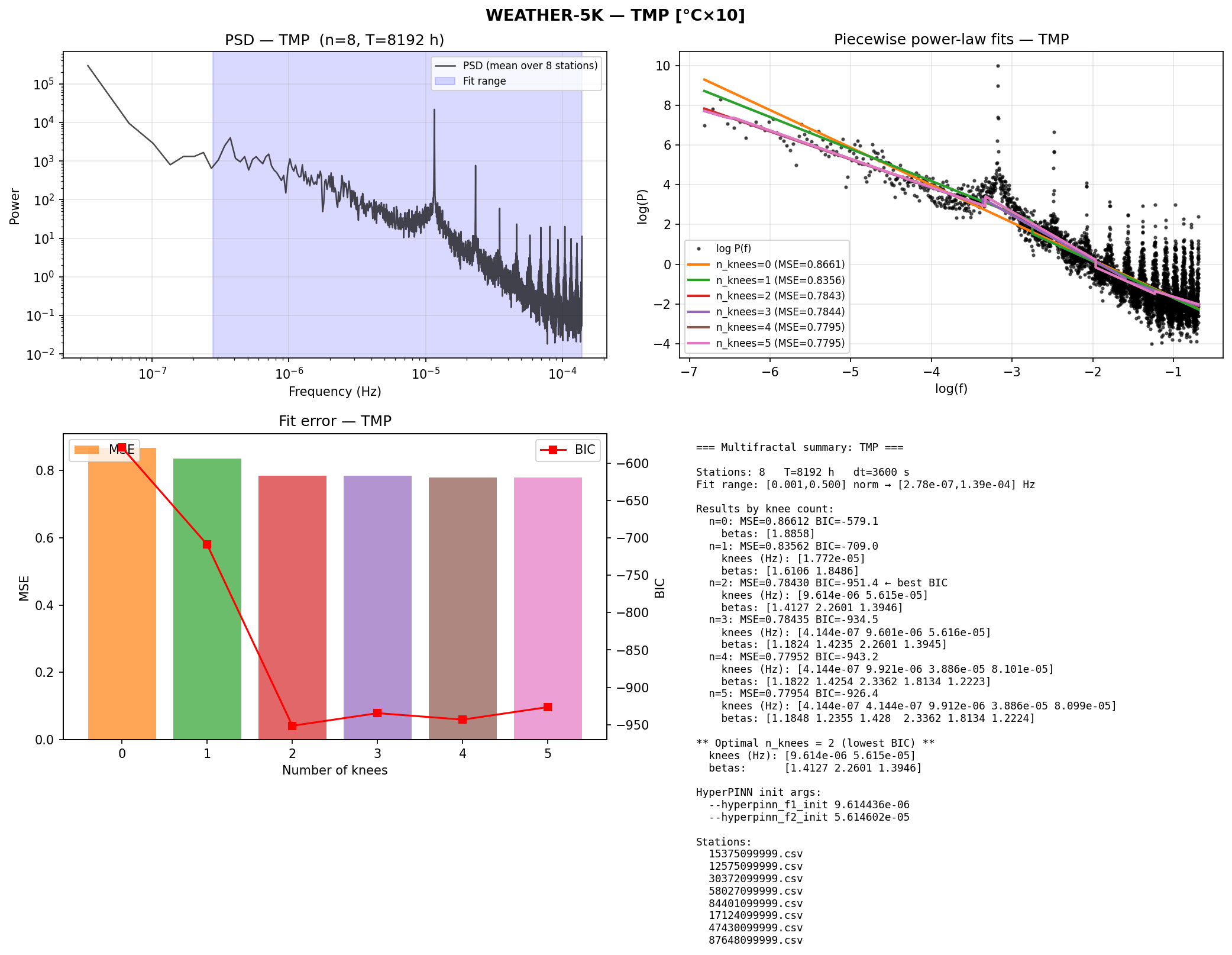}
\end{minipage}
\hfill
\begin{minipage}[b]{0.49\linewidth}
  \centering
  \includegraphics[width=\linewidth]{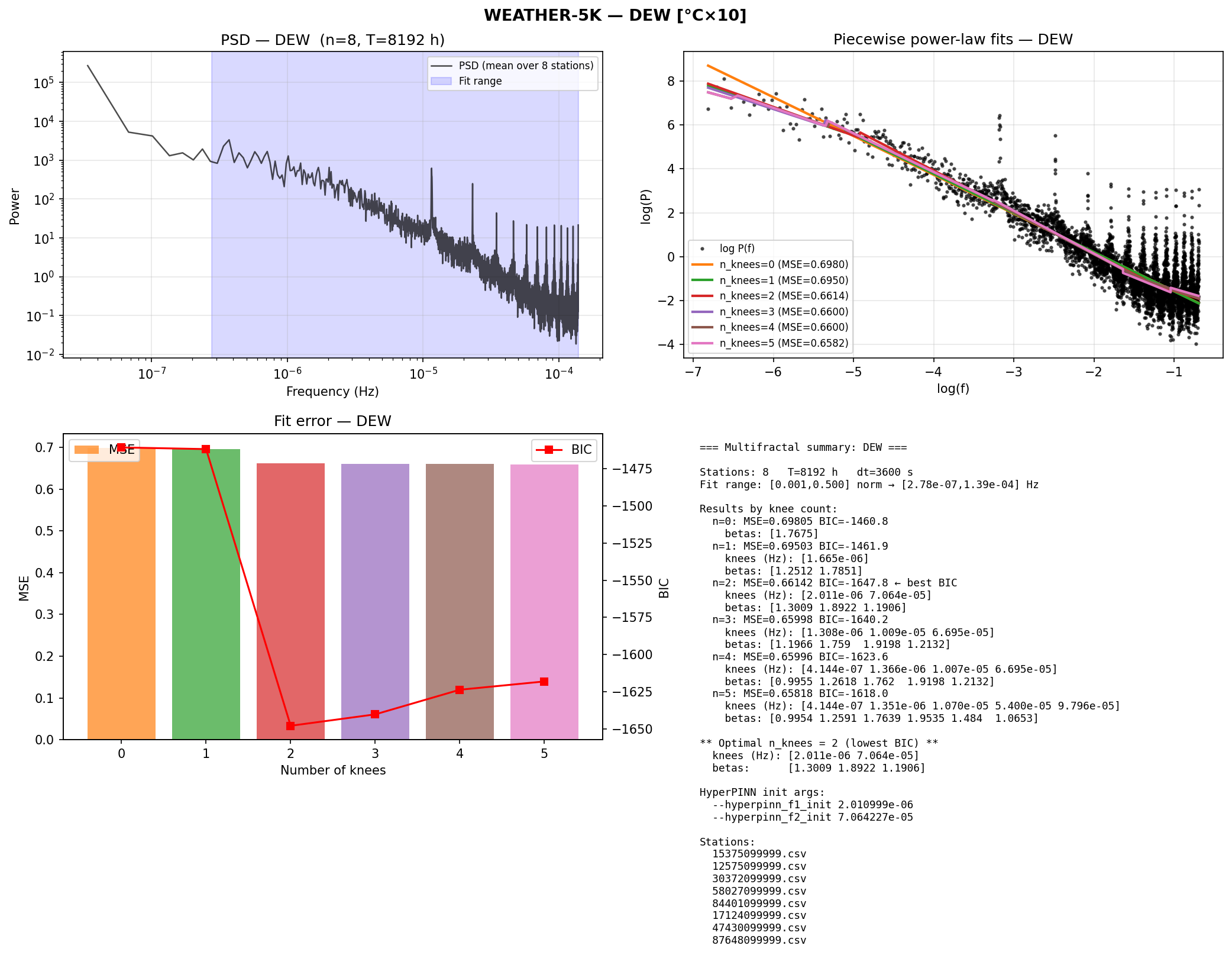}
\end{minipage}

\vspace{0.5em}

\begin{minipage}[b]{0.49\linewidth}
  \centering
  \includegraphics[width=\linewidth]{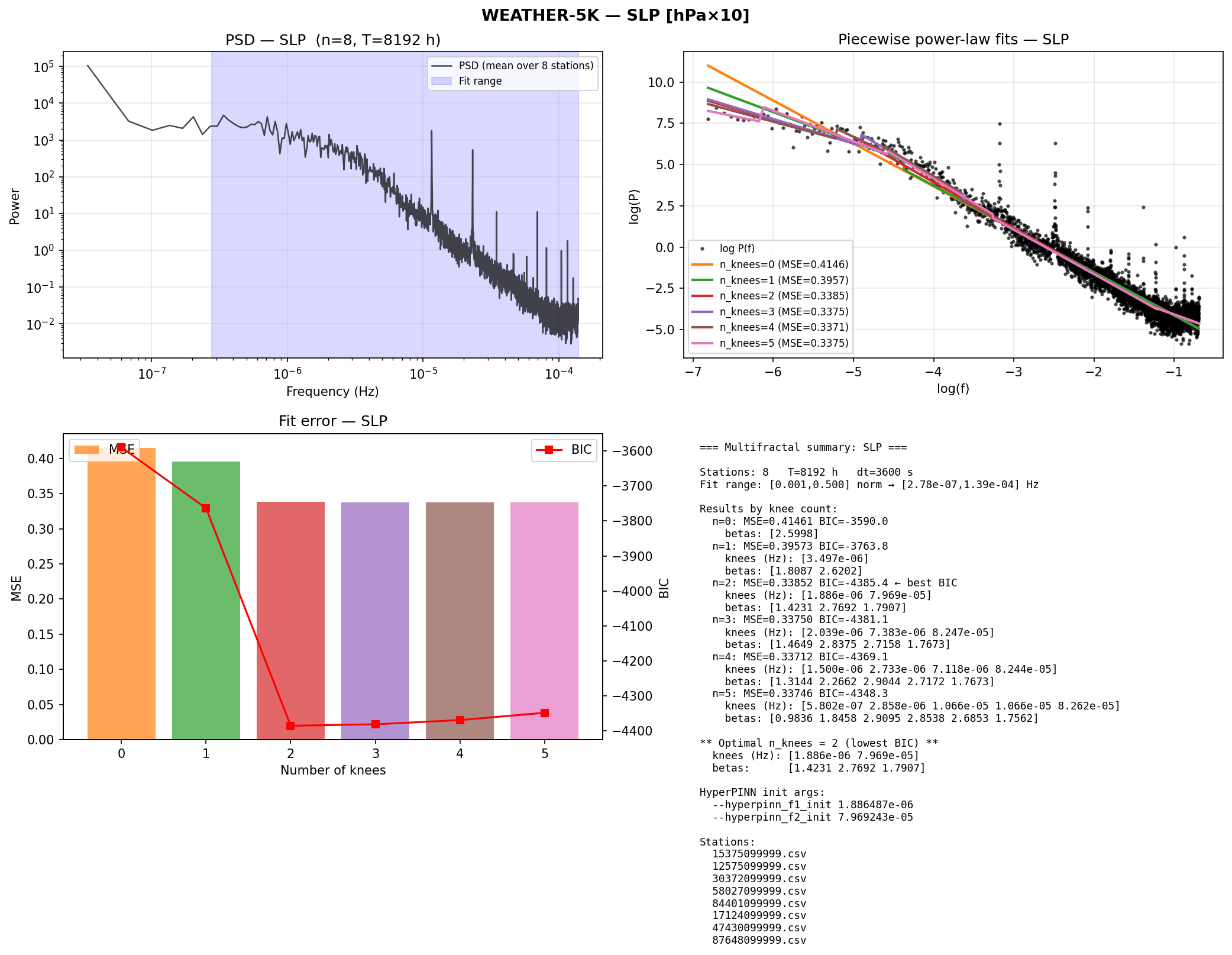}
\end{minipage}
\hfill
\begin{minipage}[b]{0.49\linewidth}
  \centering
  \includegraphics[width=\linewidth]{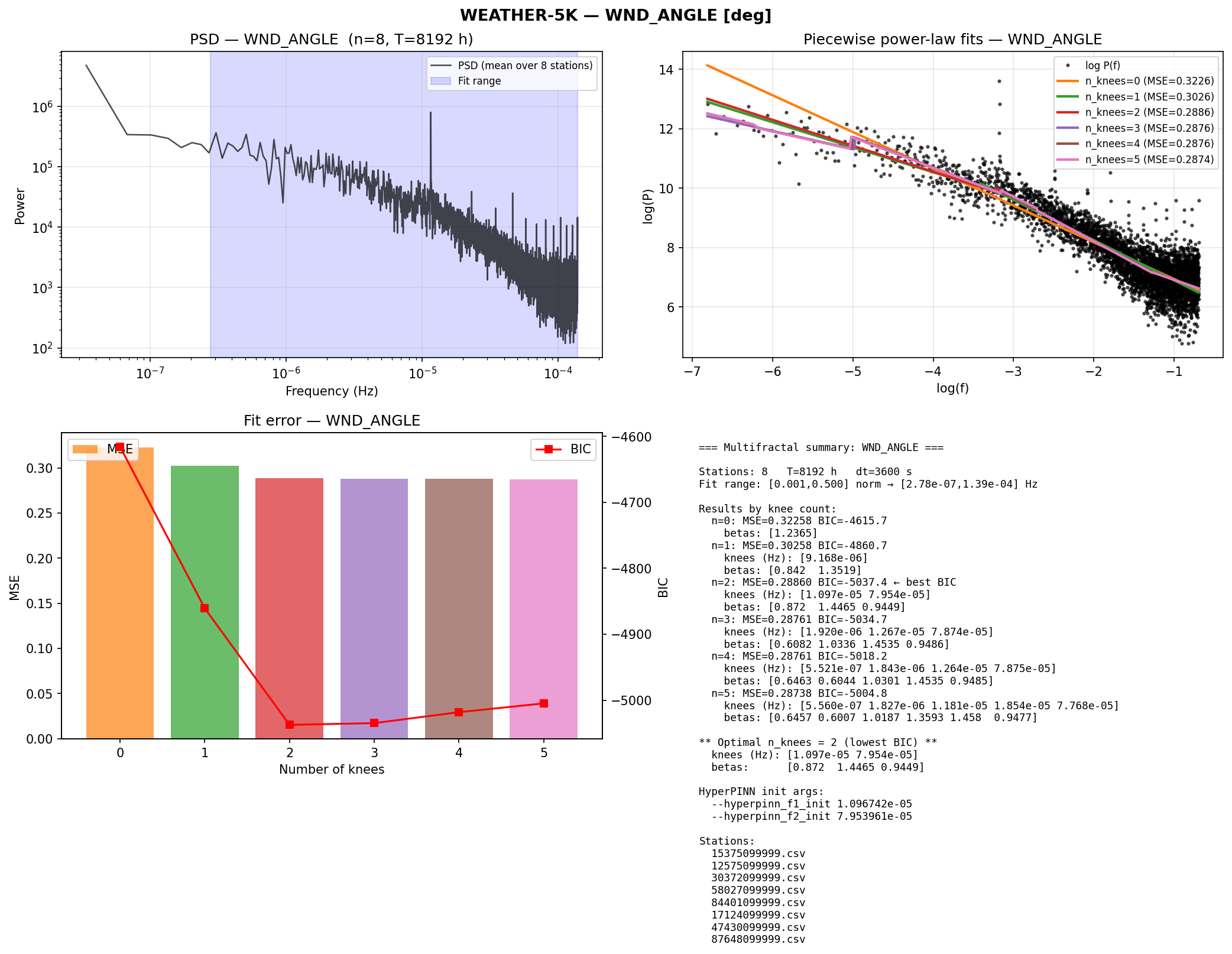}
\end{minipage}

\vspace{0.5em}

\begin{minipage}[b]{0.49\linewidth}
  \centering
  \includegraphics[width=\linewidth]{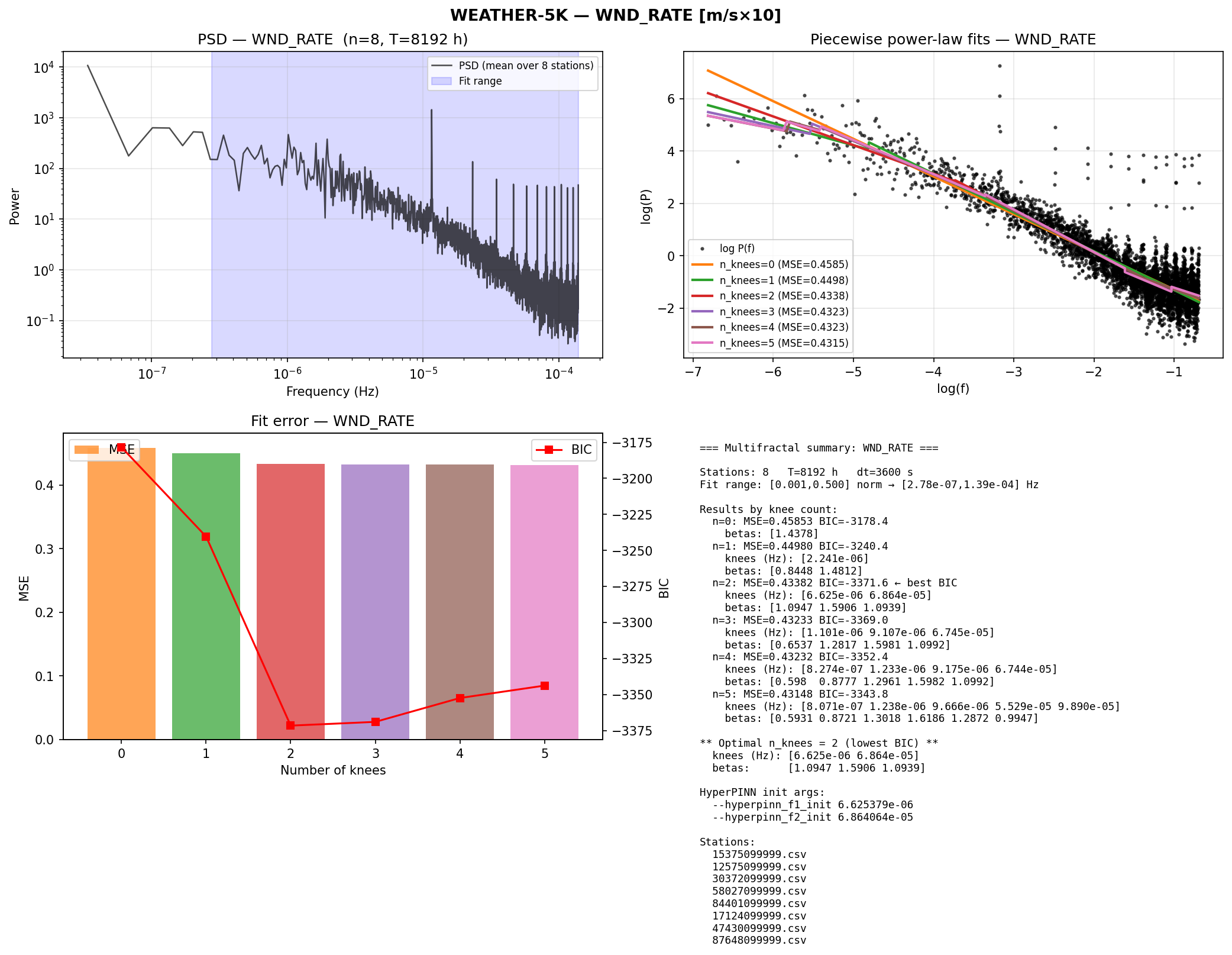}
\end{minipage}

\caption{
\textbf{Piecewise power-law structure of empirical PSD for Weather-5K (per variable).}
Weather-5K meteorological time series: temperature (tmp), dew point (dew), sea-level pressure (slp), wind angle (wnd\_angle), and wind rate (wnd\_rate).
Because PALM fits knee frequencies \emph{per variable}, each meteorological channel is analyzed independently rather than aggregated.
The piecewise power-law structure is consistent across all five variables, indicating that the spectral scaling assumption generalises beyond neural data.
In each panel, the fitted piecewise model (dashed lines) accurately captures the observed spectral slope changes at the estimated knee frequencies.
}
\label{fig:appx:psd_weather}
\end{figure}

We present empirical power spectral densities (PSDs) for fMRI cohorts and Weather-5K meteorological variables (Figures~\ref{fig:appx:psd_neural} and~\ref{fig:appx:psd_weather}).
Across all settings, the log-log PSD exhibits a piecewise linear structure with distinct scaling exponents $\beta_k$ separated by knee frequencies $(f_1, f_2)$, consistent with the piecewise power-law model in Equation~\ref{eq:piecewise_powerlaw}.
For the Weather-5K dataset, we fit the piecewise model \emph{independently per meteorological variable} (temperature, dew point, sea-level pressure, wind angle, wind rate), since each channel has its own characteristic spectral profile; PALM therefore infers a separate set of knee frequencies and scale-specific $\Delta_k$ for each variable.
This empirical regularity motivates the SpectralHyperNet-based scale inference at the heart of PALM: rather than imposing a single-exponent scaling law, we fit a configurable piecewise model (default $K=3$ in the main experiments) whose knee frequencies adaptively parameterize the scale-specific discretization parameters $\Delta_k$.

\section{Training Details}
\label{sec:appx:training_details}

For HCP experiments, all models are trained for 35 epochs using AdamW ($\mathrm{lr} = 10^{-3}$, weight decay $= 0.1$) with cosine annealing and linear warmup. Gradient norms are clipped to 1.0. Label smoothing ($\epsilon = 0.1$) is applied for classification. Batch size is 16 for the brain-state shift experiment (Axis iii) and 32 for motor decoding (Axes i--ii). All HCP tables and analyses use five random seeds. We report mean $\pm$ standard deviation of test accuracy. Each HCP fMRI run is trained on a single NVIDIA RTX A5000 GPU (24GB) on the current server.

Weather-5K forecasting uses the same parameter-matched backbone family but a forecasting-specific optimization setup: 30 epochs, AdamW with learning rate $5\times10^{-5}$, weight decay $0.05$, cosine scheduling with 5 warmup epochs, batch size 2048, and early stopping on validation MSE. The input window is 48 hourly observations and the prediction head outputs $H\times6$ values for horizons $H\in\{24,72,120,168\}$ hours. Training uses MSE in the globally normalized target space; reported MAE is converted back to original variable units for interpretability. Because Weather-5K is substantially larger than HCP fMRI, Weather-5K runs are trained with distributed data parallelism on four 80GB NVIDIA H100 GPUs.

\textbf{Weather-5K RevIN.}
For Weather-5K, we apply a stateless reversible instance normalization (RevIN)~\citep{kim2021reversible} to reduce station-level mean/variance shortcuts under held-out-station spatial OOD. After the dataset-level training-statistic z-score normalization, RevIN computes the per-window, per-variable mean and standard deviation over time, feeds $(x-\mu)/\sigma$ to the sequence backbone, and denormalizes the $H$-step forecast with the same $(\mu,\sigma)$ before the MSE loss and evaluation metrics are computed. SpectralHyperNet receives the pre-RevIN signal (or an explicit longer spectral context when provided), so PSD-based knee estimation is not driven by the window-wise amplitude removal introduced by RevIN.

The total training objective for PALM is:
\begin{equation}
\mathcal{L} = \mathcal{L}_{\mathrm{task}} + w_{\mathrm{fit}}\,\mathcal{L}_{\mathrm{fit}} + w_{\mathrm{seam}}\,\mathcal{L}_{\mathrm{seam}} + \lambda_\Delta\,\mathcal{L}_{\Delta} + \lambda_A\,\mathcal{L}_A + w_{\mathrm{hyp}}\,\mathcal{L}_{\mathrm{hyp}} + \lambda_\beta\,\mathcal{L}_{\beta},
\end{equation}
where:
\begin{itemize}
    \item $\mathcal{L}_{\mathrm{task}}$: cross-entropy classification loss (with label smoothing $\epsilon=0.1$);
    \item $\mathcal{L}_{\mathrm{fit}}$: multi-scale MAE spectral fit loss ($w_{\mathrm{fit}} = 3.0$), averaged equally across all ROIs;
    \item $\mathcal{L}_{\mathrm{seam}}$: piecewise continuity (seam) loss penalizing log-PSD discontinuities at $f_1, f_2$ ($w_{\mathrm{seam}} = 0.1$);
    \item $\mathcal{L}_{\Delta} = \mathbb{E}[\log(\Delta/\mathrm{TR})^2]$: TR-anchoring regularization ($\lambda_\Delta = 0.1$);
    \item $\mathcal{L}_A = \sum_b (\log\overline{|A|}_b - \log a_{0,b})^2$, with $\overline{|A|}_b$ the mean $|A|$ over the heads of block $b$ and $a_{0,b}$ its value at initialization: $A$-scale regularization ($\lambda_A = 0.1$);
    \item $\mathcal{L}_{\mathrm{hyp}}$: auxiliary spectral alignment loss ($w_{\mathrm{hyp}} = 0.3$);
    \item $\mathcal{L}_{\beta}$: extreme-$\beta$ regularization penalizing exponents approaching $[0.3, 5.0]$ boundaries ($\lambda_{\beta} = 0.5$).
\end{itemize}
The acquisition-unit anchoring of $\Delta$, $\tau$-initialization of $A$, and $\mathcal{L}_A$ constitute the default PALM protocol across all experiments unless an ablation removes a component.

\section{Architecture Details}
\label{sec:appx:arch_details}

\begin{algorithm}[t]
\centering
\footnotesize
\caption{Dual physical anchoring in PALM: spectrum-derived $\Delta$ and memory-timescale $A$.}
\label{alg:spectrum_to_delta}
\begin{tabular}{p{0.96\linewidth}}
\toprule
\textbf{Input:} time-domain sequence $x_{1:T}$ for the SSM backbone; spectral context $x^{\mathrm{spec}}$ (the same sequence by default, or the full 284-TR motor run in early-window HCP experiments); frequency bounds $[f_{\min}, f_{\max}]$; number of scales $K$ (default $K=3$); memory-timescale range $[\tau_{\min},\tau_{\max}]$ (acquisition units, $1$--$100\times$TR). \\
\midrule
\textbf{1. PSD estimation:} compute $P(f)=|\mathrm{FFT}(x^{\mathrm{spec}})|^2$ over non-DC frequencies in $[f_{\min}, f_{\max}]$. \\
\textbf{2. Piecewise spectral fit:} SpectralHyperNet predicts $K-1$ knees and $K$ exponents; a differentiable piecewise power-law $\hat P(f)\propto f^{-\beta_k}$ is fit to $P(f)$ with spectral-fit and seam-continuity losses. \\
\textbf{3. Representative frequencies:} the knees define $K$ frequency bands; for each band, compute $f_{c,k}=\int fP_k(f)df/\int P_k(f)df$ under the fitted exponent $\beta_k$. \\
\textbf{4. Anchor $\Delta$ (spectrum $\to$ acquisition units):} map each $f_{c,k}$ to $\Delta_k$ using $\Delta_k \propto f_{c,k}$; by default, $m_k=p_k(1-w)+w\,t_k$ with $w=0.3$ and $\Delta_k=\Delta_{\min}+(\Delta_{\max}-\Delta_{\min})m_k$, followed by an ordered fast-to-slow constraint; a log-scale penalty $\mathcal{L}_\Delta$ holds each $\Delta_k$ in acquisition units. \\
\textbf{5. Anchor $A$ (memory timescales):} initialize each head's state matrix $A=-1/\tau$ by drawing $\tau\sim\mathrm{LogUniform}(\tau_{\min},\tau_{\max})$ \emph{per head}, so every memory constant is pinned inside the physical range $[\tau_{\min},\tau_{\max}]$ (the scale separation itself is carried by the spectrum-derived $\Delta_k$ of step~4, not by $A$); a coupling penalty $\lambda_A\mathcal{L}_A$, $\mathcal{L}_A=\sum_b(\log\overline{|A|}_b-\log a_{0,b})^2$, pins $|A|$ near its initialization so $A$ cannot rescale to undo the spectrum-derived $\Delta_k$ (gauge fixing, Cor.~\ref{cor:gauge_fix}). \\
\textbf{6. Time-domain SSM processing:} process the original sequence $x_{1:T}$ without splitting or filtering it; SSM heads assigned to scale $k$ share $\Delta_k$ (each head keeping its own anchored $A$), and their scale-specific states interact through cross-scale attention before prediction. \\
\bottomrule
\end{tabular}
\end{algorithm}

\subsection{Energy-Weighted Center Frequency}

The knees define frequency bands, but PALM does not use the geometric midpoint of each band as its scale. Instead, for a band $[f_a,f_b]$ with fitted slope $P_k(f)\propto f^{-\beta_k}$, we compute the energy centroid
\[
f_{c,k}=\frac{\int_{f_a}^{f_b} f\,P_k(f)\,df}{\int_{f_a}^{f_b} P_k(f)\,df}.
\]
This centroid is evaluated in closed form, so low-frequency-dominated bands yield lower representative frequencies than midpoint rules, while flatter bands move the representative frequency toward the band center.

\subsection{Multi-Scale State-Space Block}

Each PALM block extends the Mamba-2 Structured State Space Duality (SSD) kernel to a multi-scale setting.

\textbf{Backbone and prediction head.} The backbone consists of four stacked PALM blocks with intermediate gated MLP layers. The model dimension is $d_{\mathrm{model}} = 320$ with an intermediate MLP dimension of $d_{\mathrm{inter}} = 1024$. Temporal mean pooling is applied to the final hidden sequence, followed by a task-specific linear classifier (binary for HCP sex/brain-state shift, 5-class for motor decoding, regression/forecasting for Weather-5K). Baselines are sized to a common budget: Mamba2 (7.72M; $d_{\mathrm{model}}=352$, 9 layers), Transformer (7.66M; $d_{\mathrm{model}}=280$, 8 layers), and 1D-CNN (7.56M; channels 220/440/660/880). PALM totals 5.98M parameters (including the 0.78M SpectralHyperNet, 13\% of the model); the PALM ablations (learnable/random $\Delta$) share this budget.

\textbf{Head grouping.} We employ $H = 6$ SSM heads grouped into $K = 3$ scale groups, with two heads per scale. Heads within the same scale share the discretization parameter $\Delta_k$ but maintain independent state transition matrices.

\textbf{State transition initialization.} The SSM $A$ matrix is initialized as $A = -1/\tau$ with $\tau \sim \mathrm{LogUniform}(1 \times \mathrm{TR}, 100 \times \mathrm{TR})$, reflecting the physiologically motivated time-constant range of BOLD dynamics. An $A$-scale regularization $\mathcal{L}_A$ is applied to prevent $A$ from drifting away from this initialization to compensate for $\Delta$, ensuring that the joint constraint on $\tau_\text{eff}$ remains effective throughout training.

\textbf{State evolution.} For each head $j$ assigned to scale $k(j)$, the discretized state update is:
\begin{equation}
\mathbf{h}^{(j)}_t = \bar{A}^{(j)}(\Delta_{k(j)}) \mathbf{h}^{(j)}_{t-1} + \bar{B}(\Delta_{k(j)}) \mathbf{x}_t.
\end{equation}
Input projections are shared across heads, while state transitions differ to capture scale-specific decay dynamics.

\textbf{Scale aggregation.} Within each scale group, head outputs are averaged:
\begin{equation}
\mathbf{h}^{(k)}_t = \frac{1}{|J_k|} \sum_{j \in J_k} \mathbf{h}^{(j)}_t.
\end{equation}

\subsection{Cross-Scale Fusion}

Scale-specific representations are concatenated and passed through a lightweight cross-scale attention module:
\begin{align}
\mathbf{h}_t &= \mathrm{Concat}(\mathbf{h}^{(1)}_t, \mathbf{h}^{(2)}_t, \mathbf{h}^{(3)}_t), \\
\mathrm{Attn}(Q, K, V) &= \mathrm{softmax}\!\left( \frac{QK^\top}{\sqrt{d}} \right) V.
\end{align}
This allows high- and low-frequency information to be selectively combined while preserving structured multi-scale dynamics.

\section{Delta Mapping Ablation: Mixing Weight and Normalization Mode}
\label{sec:appx:delta_ablation}

The per-band frequency-to-delta mapping combines global band position $p_k$ and within-band normalized frequency $t_k$ via $m_k = p_k(1-w) + w\,t_k$. To assess sensitivity to this design choice, we ablate (i) the mixing weight $w \in \{0, 0.3, 0.5, 1.0\}$ and (ii) normalization mode (global vs.\ per-band). At $w=0$, $m_k = p_k$ (band position only); at $w=1$, $m_k = t_k$ (within-band variation only); the default $w=0.3$ balances both. Global normalization maps $f_{c,k}$ directly onto $[\Delta_{\min}, \Delta_{\max}]$ without per-band boundaries. These ablations are designed to verify that the core principle---higher $f \to$ larger $\Delta$---drives the benefit, rather than the specific mixing formula.

\section{HCP Motor Decoding under Temporal-Context Shift: Additional Details}
\label{sec:appx:hcp_tr_shift}

This appendix provides supporting detail for Section~\ref{sec:results:hcp_tr_shift}. The severe 2TR result is reported in the main text (Table~\ref{tab:drift2tr}); here we include the moderate 3TR table, drift intervention tables, scaling figures, and Figure~\ref{fig:tr_shift_stability}. As in the main text, early-window models are trained and evaluated on first-2/3TR motor-block inputs, while SpectralHyperNet uses the full 284-TR motor run for spectral calibration.

\paragraph{Full block $\rightarrow$ 3 TRs (moderate truncation).}
Under the milder full-block $\rightarrow$ 3 TRs setting, PALM achieves the best 3TR accuracy (Table~\ref{tab:drift3tr}). At 3TR, the learnable-$\Delta$ ablation has the highest CKA and dCor, and PALM exceeds all baseline architectures on both.
The same paired-seed pattern holds at 3TR: PALM improves over Mamba2 on accuracy, CKA, and dCor for all five seeds, consistent with the robustness trend.

\paragraph{Drift intervention and scaling.}
Explicit drift regularization on Mamba2 improves stability without degrading ID performance, providing evidence that representation drift is a useful intervention target; PALM achieves comparable stability through architecture alone (detailed tables and scaling figures in Appendix~\ref{sec:appx:hcp_tr_shift}).
Complementary WeightWatcher analysis (Appendix~\ref{sec:appx:weightwatcher}) shows that PALM preserves heterogeneous layer-wise power-law exponents under full-block and 2TR inputs, whereas Mamba2 concentrates near a saturated regime, supporting the view that PALM resists representation homogenization rather than merely improving output accuracy.

\paragraph{Realized memory timescales.}
Over five seeds, the full model has a median realized $\tau\approx5.9$\,TR, with heads spanning $\approx1$--$30$\,TR. Without the anchoring, the median $\tau$ is $0.20$\,TR (learnable $\Delta$) and $0.17$\,TR (random $\Delta$).

\paragraph{Stability statistics for full-run PALM vs the scale-space grid.}
CKA is $0.847$ vs $0.833$ (difference $+0.014$, paired-bootstrap 95\% CI $[+0.005,+0.029]$). dCor is $0.871$ vs $0.859$ (difference $+0.012$, 95\% CI $[+0.005,+0.019]$, $d_z=1.32$).

\paragraph{Notes on the causal decomposition (Table~\ref{tab:causal_variants}).}
The full-run PALM CKA in the causal decomposition ($0.847\pm0.023$) comes from the causal-decomposition runs, and Table~\ref{tab:drift2tr} reports $0.812\pm0.033$ from the main runs.
At 2\,TR the window is below the minimum length for a PSD fit, so the prefix-only variant falls back to the raw truncated window (a logged switch); it tests ``no privileged information,'' not ``a spectrum estimated from two samples.''

\begin{table}[htbp]
\centering
\footnotesize
\caption{
\textbf{HCP motor decoding under moderate temporal-context truncation (full block $\rightarrow$ 3 TRs).}
Values are mean $\pm$ std over five seeds.
}
\label{tab:drift3tr}
\begin{tabular}{lcccc}
\toprule
Model & Full block & 3 TRs & CKA (Full block$\leftrightarrow$3TR) & dCor (Full block$\leftrightarrow$3TR) \\
\midrule
Mamba2 & $0.983\pm0.002$ & $0.931\pm0.010$ & $0.786\pm0.031$ & $0.817\pm0.025$ \\
1DCNN & $0.978\pm0.007$ & $0.874\pm0.001$ & $0.585\pm0.008$ & $0.658\pm0.001$ \\
Transformer & $0.982\pm0.006$ & $0.927\pm0.007$ & $0.839\pm0.019$ & $0.864\pm0.016$ \\
PALM (learnable $\Delta$) & $0.982\pm0.001$ & $0.937\pm0.004$ & $\mathbf{0.865\pm0.006}$ & $\mathbf{0.881\pm0.005}$ \\
PALM (random $\Delta$) & $0.983\pm0.004$ & $0.934\pm0.003$ & $0.846\pm0.026$ & $0.873\pm0.019$ \\
\midrule
\textbf{PALM} & $\mathbf{0.987\pm0.007}$ & $\mathbf{0.944\pm0.005}$ & $0.854\pm0.022$ & $0.878\pm0.020$ \\
\bottomrule
\end{tabular}
\end{table}

\paragraph{Drift intervention.}
We introduce explicit drift regularization for Mamba2: $\mathcal{L} = \mathcal{L}_{\text{task}} + \lambda D(z(x^{\text{full}}), z(x^{\text{trunc}}))$. This model (Mamba2+drift) improves representation stability and decoding accuracy under temporal-context truncation without degrading ID performance, supporting representation drift as a useful robustness target. PALM achieves comparable stability through architecture alone, without explicit shift supervision. See Tables~\ref{tab:drift_intervention_2tr} and~\ref{tab:drift_intervention_3tr}.

\paragraph{Scaling behavior.}
Larger PALM models exhibit increasingly stable representations under temporal-context truncation, whereas the single-scale SSM baseline shows smaller robustness gains from additional capacity (Figure~\ref{fig:scaling}). Additional WeightWatcher analysis (Section~\ref{sec:appx:weightwatcher}) provides complementary evidence that multi-scale structure helps preserve representation diversity.

\begin{figure}[htbp]
\centering
\includegraphics[width=\textwidth]{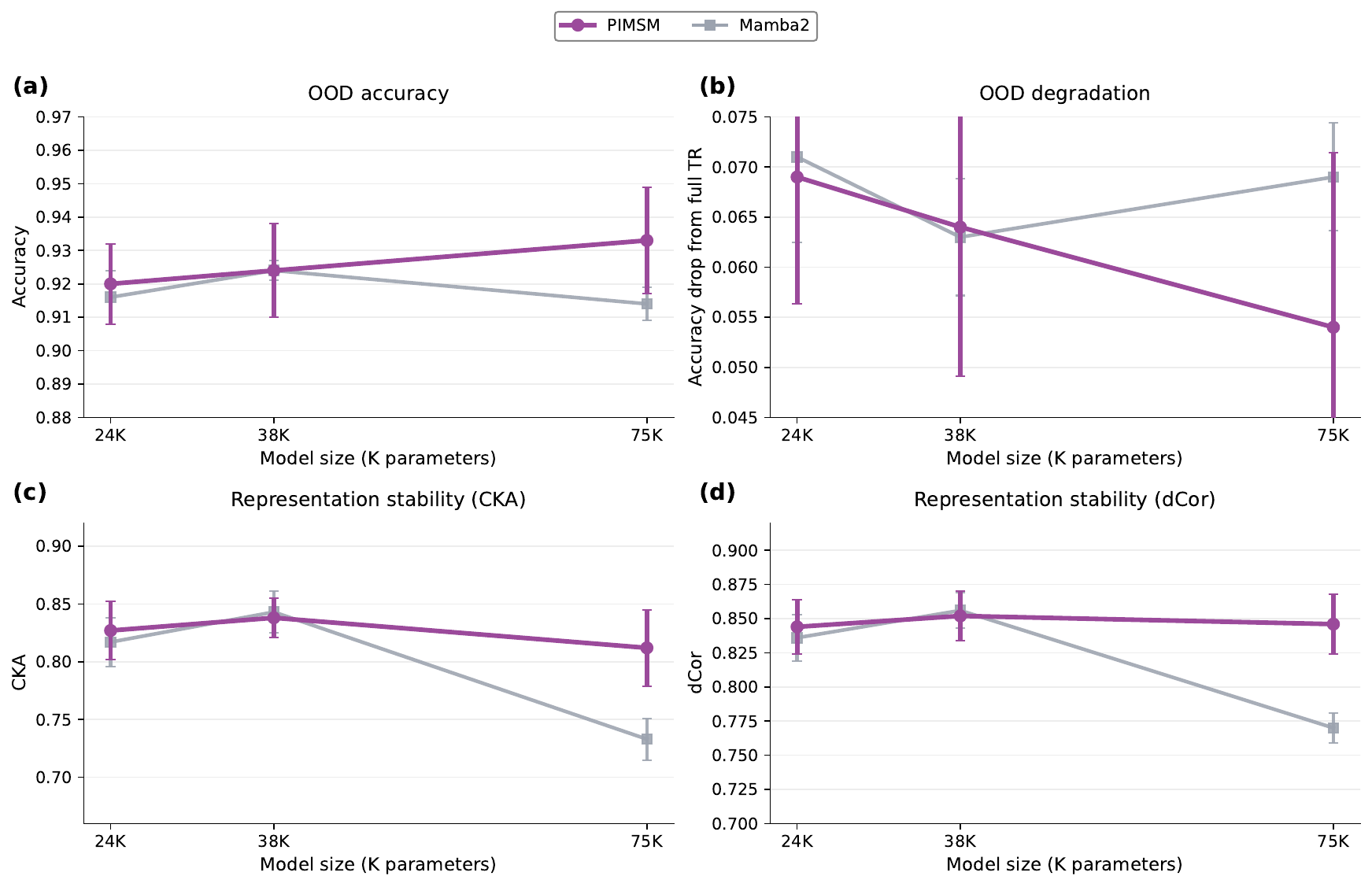}
\caption{
\textbf{Scaling behavior under temporal-context truncation (full block $\rightarrow$ 2TR).}
(a) Decoding accuracy vs model size. (b) Accuracy gap relative to full-block input. (c--d) CKA and dCor between full-block and 2TR representations.
}
\label{fig:scaling}
\end{figure}

\begin{table}[htbp]
\centering
\footnotesize
\caption{Drift intervention: full block $\rightarrow$ 2 TR. Values are mean $\pm$ std over five seeds.}
\label{tab:drift_intervention_2tr}
\begin{tabular}{lcccc}
\toprule
Model & Full block & 2 TRs & CKA (Full block$\leftrightarrow$2TR) & dCor (Full block$\leftrightarrow$2TR) \\
\midrule
PALM & 0.987$\pm$0.007 & 0.933$\pm$0.016 & 0.812$\pm$0.033 & 0.846$\pm$0.022 \\
Mamba2+drift & 0.987$\pm$0.004 & 0.923$\pm$0.009 & 0.837$\pm$0.021 & 0.854$\pm$0.019 \\
\bottomrule
\end{tabular}
\end{table}

\begin{table}[htbp]
\centering
\footnotesize
\caption{Drift intervention: full block $\rightarrow$ 3 TR. Values are mean $\pm$ std over five seeds.}
\label{tab:drift_intervention_3tr}
\begin{tabular}{lcccc}
\toprule
Model & Full block & 3 TRs & CKA (Full block$\leftrightarrow$3TR) & dCor (Full block$\leftrightarrow$3TR) \\
\midrule
PALM & 0.987$\pm$0.007 & 0.944$\pm$0.005 & 0.854$\pm$0.022 & 0.878$\pm$0.020 \\
Mamba2+drift & 0.987$\pm$0.004 & 0.939$\pm$0.011 & 0.854$\pm$0.036 & 0.879$\pm$0.025 \\
\bottomrule
\end{tabular}
\end{table}

\section{WeightWatcher Analysis: Layer-wise Power-Law Exponents}
\label{sec:appx:weightwatcher}

\begin{figure}[htbp]
\centering
\includegraphics[width=0.85\linewidth]{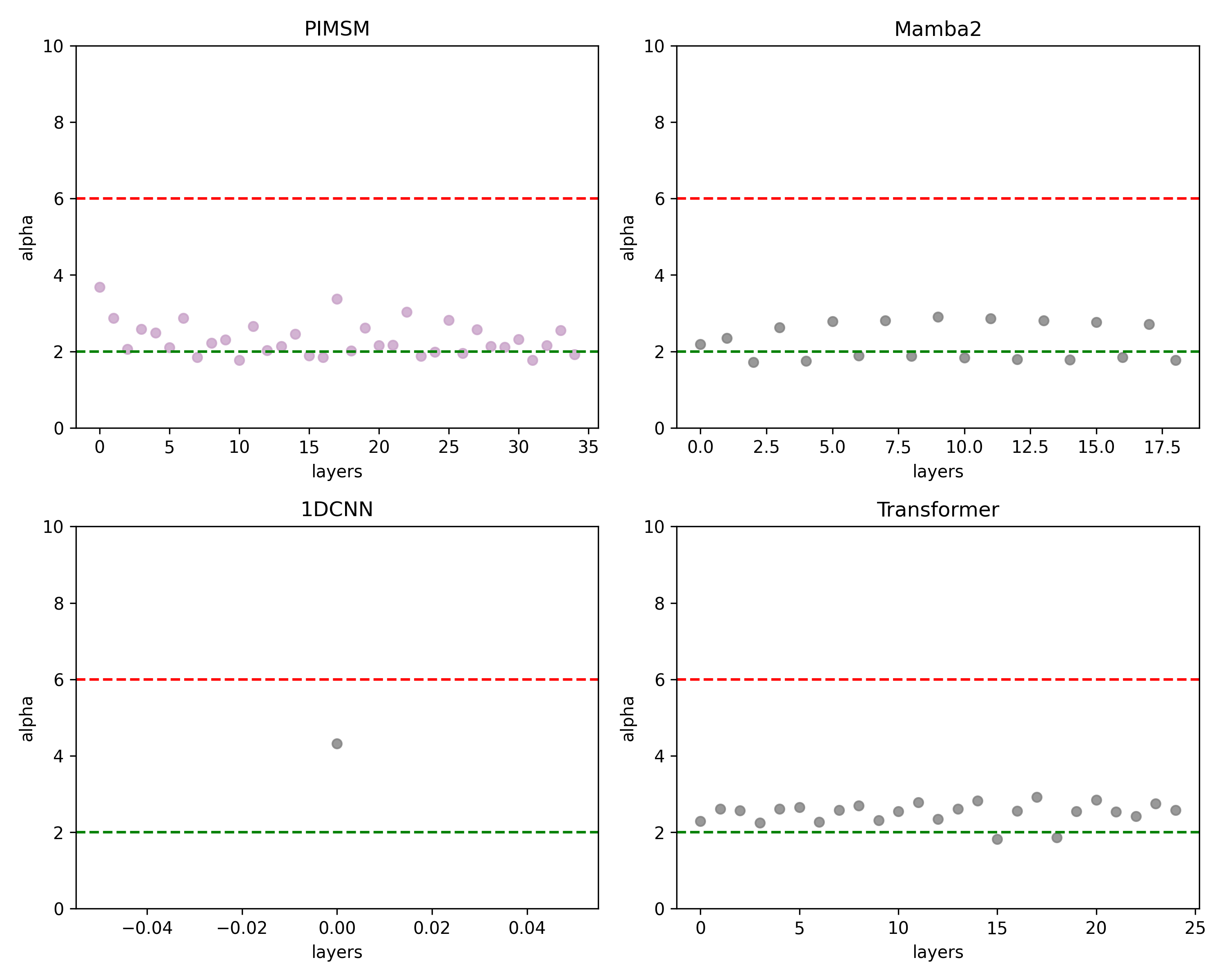}
\caption{
\textbf{WeightWatcher layer-wise $\alpha$ distributions under full-block input (HCP motor task fMRI).}
Each point represents the power-law exponent $\alpha$ of the empirical spectral density of a weight matrix.
The red dashed line ($\alpha = 6$) marks the boundary above which weights are statistically indistinguishable from random matrices;
the green dashed line ($\alpha = 2$) indicates the onset of over-training.
PALM exhibits heterogeneous $\alpha$ values (1.77--3.69) spread across a wide range, indicating that different layers operate at distinct feature-complexity regimes.
Mamba2 and Transformer concentrate near $\alpha \approx 2$, suggesting that most layers have reached training saturation.
1D-CNN has only one analyzable layer (remaining layers contain weight matrices too small for reliable spectral fitting).
}
\label{fig:ww_full_tr}
\end{figure}

\begin{figure}[htbp]
\centering
\includegraphics[width=0.85\linewidth]{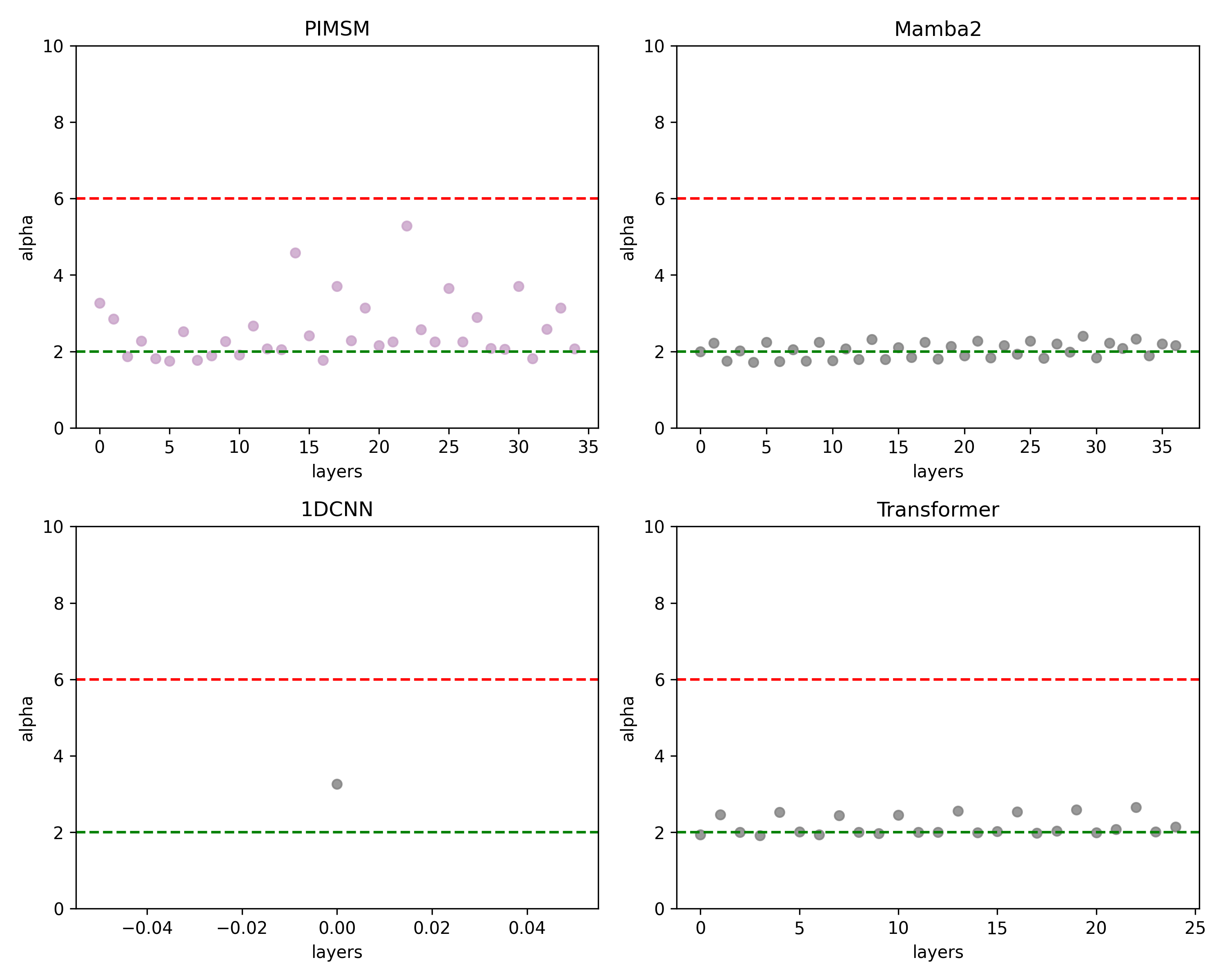}
\caption{
\textbf{WeightWatcher layer-wise $\alpha$ distributions under 2 TR early-window condition (extreme temporal truncation).}
Under the 2 TR condition---an extreme data-scarce scenario in which only two time points are available per trial---PALM maintains a wide $\alpha$ spread (1.75--5.29), indicating that layer-wise representational diversity is preserved despite the severely reduced input.
In contrast, Mamba2 concentrates more uniformly near $\alpha \approx 2$ (all 37 analyzable layers saturate simultaneously), suggesting greater representation homogenization when learning signals are scarce.
The Transformer retains a similar distribution to its full-block counterpart but with reduced variance.
}
\label{fig:ww_2tr}
\end{figure}

WeightWatcher~\citep{martin2021predicting} quantifies the implicit self-regularization of each weight matrix by fitting the tail of its empirical spectral density (ESD) to a power law $\rho(\lambda) \propto \lambda^{-\alpha}$.
Matrices with $\alpha \in [2, 6]$ are considered \emph{well-trained}: $\alpha < 2$ indicates over-training (insufficient regularization), and $\alpha > 6$ indicates that the matrix is statistically indistinguishable from a random initialization.
Note that WeightWatcher analyzes 2D matrices only; 4D convolution tensors (as in 1D-CNN) are excluded except for the final linear layer.

\paragraph{Full-trial condition (Figure~\ref{fig:ww_full_tr}).}
Under the standard full-block condition, PALM exhibits $\alpha \in [1.77, 3.69]$ with notable layer-to-layer variation.
This heterogeneity reflects the multi-scale inductive bias of the architecture: different PALM layers specialize to capture temporal correlations at different scales, resulting in layers with distinct implicit regularization properties.
In contrast, Mamba2 and Transformer accumulate near $\alpha \approx 2.0$, indicating that most layers are in a saturated learning state.
These models may be efficiently encoding in-distribution patterns but leave little representational capacity available for adaptation under shift.

\paragraph{2 TR early-window condition (Figure~\ref{fig:ww_2tr}).}
Under the extreme 2 TR condition, PALM's $\alpha$ distribution widens further (1.75--5.29).
Rather than homogenizing under reduced input, the multi-scale architecture preserves layer-wise role differentiation---deeper layers tend toward higher $\alpha$ (less specialized), while early layers remain well-fitted---suggesting that the model gracefully degrades without global collapse.
Mamba2, by contrast, narrows its $\alpha$ distribution toward $\approx 2.0$ with near-zero variance across all 37 analyzable layers, consistent with representation homogenization under information-sparse inputs.
The Transformer shows moderate variance reduction but remains closer to its full-block distribution, perhaps due to the global attention mechanism retaining some diversity.

\paragraph{Interpretation.}
These layer-wise spectral analyses provide weight-level corroboration for the conclusion drawn from CKA and distance correlation measurements:
PALM's multi-scale structure actively resists representation collapse when temporal context is severely reduced.
The preserved heterogeneity of $\alpha$ values across layers suggests that the physics-informed multi-scale parameterization may act as an implicit regularizer that discourages collapse to a single dominant feature regime under severe temporal truncation.

\section{Self-Supervised Pretraining Details}
\label{sec:appx:pretraining}

\paragraph{Uniform spatiotemporal masking.}
We adopt a uniform masking strategy applied jointly over the temporal and spatial dimensions of the input sequence.
Each $(t,r)$ location (time $t$, ROI $r$) is masked independently with a fixed probability of $10\%$.
To further increase reconstruction difficulty and encourage learning of spatial dependencies, we additionally apply ROI-level temporal masking: for each ROI, contiguous time points corresponding to approximately $30\%$ of the sequence length are randomly selected and masked.

\paragraph{SpectralHyperNet during pretraining.}
SpectralHyperNet operates on the \emph{original unmasked} input $x_{1:T}$ to estimate the piecewise power-law parameters:
\[
(f_1, f_2, \beta_1, \beta_2, \beta_3) = H_{\phi}(x_{1:T}).
\]
Importantly, the knee frequencies $(f_1, f_2)$ determine the scale-specific discretization parameters $\Delta_k$ (via the energy-weighted center frequency mapping $\Delta_k = \Psi(f_{c,k})$), while the scaling exponents $\beta_k$ serve as supervision targets for the spectral fit loss and inform the energy-weighted computation of $f_{c,k}$; they do not appear explicitly in the $\Delta_k$ formula. This design ensures that spectral estimates remain stable even when the input to the SSM backbone is partially corrupted by masking.

\paragraph{Training objective.}
The reconstruction loss is an $L_1$ loss over masked positions only:
\[
\mathcal{L}_{\text{pretrain}} = \| x_{\text{masked}} - \hat{x}_{\text{masked}} \|_1.
\]

\section{Effect of Pretraining: Finetuning and Linear Probing}
\label{sec:appx:pretrain}

\subsection{Finetuning from Resting-State Pretraining}

\begin{table}[htbp]
\centering
\footnotesize
\caption{
Effect of self-supervised pretraining on resting-state fMRI followed by finetuning on HCP motor task fMRI.
Values are mean $\pm$ std over five seeds.
Representation similarity (CKA, dCor) is measured between full-block and early-window inputs.
Pretraining improves shift robustness for both models, with PALM retaining a small but consistent advantage in representation stability.
}
\label{tab:pretrain_finetune}
\begin{tabular}{lccccc}
\toprule
Model & From Scratch & Full block & 2 TRs & CKA & dCor \\
\midrule
PALM & $0.987 \pm 0.007$ & $0.987 \pm 0.004$ & $\mathbf{0.937 \pm 0.009}$ & $\mathbf{0.867 \pm 0.004}$ & $\mathbf{0.889 \pm 0.009}$ \\
Mamba2 & $0.983 \pm 0.002$ & $0.987 \pm 0.007$ & $0.933 \pm 0.011$ & $0.847 \pm 0.024$ & $0.860 \pm 0.024$ \\
\midrule
Model & From Scratch & Full block & 3 TRs & CKA & dCor \\
\midrule
PALM & $0.987 \pm 0.007$ & $0.987 \pm 0.004$ & $\mathbf{0.949 \pm 0.012}$ & $\mathbf{0.878 \pm 0.005}$ & $\mathbf{0.899 \pm 0.002}$ \\
Mamba2 & $0.983 \pm 0.002$ & $0.987 \pm 0.007$ & $0.944 \pm 0.006$ & $0.874 \pm 0.010$ & $0.887 \pm 0.011$ \\
\bottomrule
\end{tabular}
\end{table}

We evaluate whether self-supervised pretraining on resting-state fMRI improves downstream shift robustness when models are subsequently finetuned on motor task fMRI (Table~\ref{tab:pretrain_finetune}).
Pretraining raises shifted-condition accuracy and representation similarity relative to training from scratch for both PALM and Mamba2.
Under the 2 TR condition, pretrained PALM achieves $0.937 \pm 0.009$ versus $0.933 \pm 0.011$ for pretrained Mamba2; pretrained PALM also retains higher CKA ($0.867 \pm 0.004$) and dCor ($0.889 \pm 0.009$).
Under the 3 TR condition, pretrained PALM reaches $0.949 \pm 0.012$ versus $0.944 \pm 0.006$ for Mamba2.
These margins are consistent but modest, suggesting that the primary benefit of pretraining lies in improving representation stability under shift rather than raising the absolute performance ceiling.

\subsection{Linear Probing Evaluation}

\begin{table}[htbp]
\centering
\footnotesize
\caption{
Linear probing evaluation of pretrained models on HCP motor task fMRI.
Values are mean $\pm$ std over five seeds.
The backbone is frozen and only a linear classifier is trained.
CKA and dCor are measured between full-block and early-window representations of the frozen backbone.
Both models achieve comparable full-block accuracy ($\approx 0.93$), but PALM maintains slightly higher shifted-condition accuracy under 2 TR and 3 TR conditions.
}
\label{tab:linear_probe}
\begin{tabular}{lccccc}
\toprule
Model & From Scratch & Full block & 2 TRs & CKA & dCor \\
\midrule
PALM & $0.987 \pm 0.007$ & $0.934 \pm 0.015$ & $\mathbf{0.748 \pm 0.003}$ & $0.444 \pm 0.029$ & $\mathbf{0.469 \pm 0.020}$ \\
Mamba2 & $0.983 \pm 0.002$ & $\mathbf{0.936 \pm 0.015}$ & $0.730 \pm 0.015$ & $\mathbf{0.448 \pm 0.041}$ & $0.462 \pm 0.036$ \\
\midrule
Model & From Scratch & Full block & 3 TRs & CKA & dCor \\
\midrule
PALM & $0.987 \pm 0.007$ & $0.934 \pm 0.015$ & $\mathbf{0.791 \pm 0.027}$ & $0.470 \pm 0.032$ & $\mathbf{0.489 \pm 0.026}$ \\
Mamba2 & $0.983 \pm 0.002$ & $\mathbf{0.936 \pm 0.015}$ & $0.767 \pm 0.023$ & $\mathbf{0.475 \pm 0.042}$ & $0.484 \pm 0.035$ \\
\bottomrule
\end{tabular}
\end{table}

Under strict linear probing, both PALM and Mamba2 achieve similar full-block accuracy ($\approx 0.93$).
Shifted-condition accuracy drops substantially for both models relative to the finetuning setting (Table~\ref{tab:pretrain_finetune}), reflecting the difficulty of adapting frozen representations to temporal-resolution shift via a linear readout alone.
PALM maintains a consistent advantage: $0.748 \pm 0.003$ vs.\ $0.730 \pm 0.015$ at 2 TR, and $0.791 \pm 0.027$ vs.\ $0.767 \pm 0.023$ at 3 TR.
CKA values under the frozen setting are substantially lower than under finetuning, suggesting that the full representational benefit of physics-informed pretraining is realized during downstream adaptation rather than through passively linearly separable representations.
These results are consistent with the view that multi-scale temporal modeling improves the plasticity of learned features during finetuning rather than directly encoding task geometry.

\section{Directly Measuring the Induced Kernel}
\label{sec:appx:induced_kernel}

Lemma~\ref{lem:kernel_mismatch} bounds representation error by the $L^1$ mismatch between the model-induced kernel $\tilde g$ and the target kernel $g$. To make this quantity \emph{measured} rather than inferred, we estimate the induced kernel of a trained model directly, in a way defined identically for every architecture.

\paragraph{Estimator.}
For a trained encoder producing pooled representation $z$ from input $x_{1:T}$, we define the empirical impulse-response profile
\begin{equation}
\hat g(\tau) \;=\; \mathbb{E}_x\left\|\frac{\partial z}{\partial x_{T-\tau}}\right\|, \qquad \tau = 0, 1, \dots, T-1,
\end{equation}
the discrete analogue of $\tilde g$ in Equation~\ref{eq:kernel_conv}, averaged over inputs and channels via reverse-mode autodifferentiation. The target kernel $g$ is the decay profile implied by the signal's own PSD (the piecewise power-law fit of Equation~\ref{eq:piecewise_powerlaw} mapped to a time-domain kernel). We report $\|\hat g - g\|_1$, the right-hand side of Lemma~\ref{lem:kernel_mismatch}, per model.

\paragraph{Estimator validation.}
On a toy single-exponential SSM with a known timescale $\tau_0$, the estimator recovers $\tau_0$ to tolerance, and a finite-difference cross-check of $\hat g$ agrees with the autodiff profile at correlation $>0.999$, confirming that $\hat g$ measures the intended quantity.

\paragraph{Results.}
Because this misalignment pathology is specific to \emph{selective} SSMs, where $\Delta$ and $A$ grant freedom to place effective memory away from the signal's physical timescales, the decisive comparison is against the SSM/attention family PALM extends. On Weather-5K spatial-OOD ($h24$, five seeds), $\|\hat g - g\|_1 = 0.61\pm0.05$ (PALM) versus $1.00\pm0.08$ (Mamba2) and $0.82\pm0.08$ (Transformer) --- $39\%$ below the single-scale Mamba2 PALM directly modifies and $25\%$ below the Transformer; the same ordering holds on HCP ($0.69$ vs $0.99/0.99$). The 1D-CNN is numerically lower ($0.52$). Its kernel is a fixed, compact convolution that is smooth by construction, and it has no $\Delta$--$A$ freedom. The same model has the least stable representation under truncation (Table~\ref{tab:drift2tr}). In-distribution kernel alignment is necessary, not sufficient; representation stability under shift additionally requires the physical anchoring.

\end{document}